\documentclass[11 pt]{article}

\let\labelindent\relax

\usepackage{cite}
\usepackage[utf8]{inputenc}
\usepackage{amsmath,amssymb}
\usepackage{amsthm}
\usepackage{amssymb,graphicx,verbatim,enumitem,mdframed,mathrsfs}
\usepackage{mdframed}
\usepackage{subcaption}
\usepackage[toc]{appendix}

\usepackage{geometry}
\newtheorem{theorem}{\textbf{Theorem}}
\newtheorem{lemma}{\textbf{Lemma}}
\newtheorem{claim}{\textbf{Claim}}

\providecommand{\iprod}[2]{\ensuremath{\left\langle #1,\,#2  \right\rangle}}
\providecommand{\norm}[1]{\ensuremath{\left\lVert#1\right\rVert }}
\providecommand{\mnorm}[1]{\ensuremath{\left\lvert#1\right\rvert}}

\title{Byzantine Fault Tolerant Distributed Linear Regression\footnote{Manuscript revised by adding; a new improved filtering technique in Section~\ref{sec:algo_2}, and convergence analysis in case of noise in Appendix~\ref{app:noise}.}}
\author{Nirupam Gupta and Nitin H. Vaidya}
\date{Department of Computer Science, \\
Georgetown University, Washington, DC 20057, USA \\
\{\emph{first-name}\}.\{\emph{last-name}\}@georgetown.edu}

\begin{document}

\maketitle

\def\R{{\mathbb R}}
\def\Z{{\mathbb Z}}
\def\N{{\mathbb N}}
\def\H{{\mathcal H}}
\def\B{{\mathcal B}}
\def\W{{\mathcal W}}
\def\F{{\mathcal F}}
\def\E{{\mathbb E}}
\def\inti{{\sf Int}}

\begin{abstract}
    This paper considers the problem of Byzantine fault tolerance in distributed linear regression in a multi-agent system. However, the proposed algorithms are given for a more general class of distributed optimization problems, of which distributed linear regression is a special case. The system comprises of a server and multiple agents, where each agent is holding a certain number of data points and responses that satisfy a linear relationship (could be noisy). The objective of the server is to determine this relationship, given that some of the agents in the system (up to a known number) are Byzantine faulty (aka. actively adversarial). We show that the server can achieve this objective, in a deterministic manner, by robustifying the original distributed gradient descent method using norm based filters, namely \emph{norm filtering} and \emph{norm-cap filtering}, incurring an additional log-linear computation cost in each iteration. The proposed algorithms improve upon the existing methods on three levels: i) no assumptions are required on the probability distribution of data points, ii) system can be partially asynchronous, and iii) the computational overhead (in order to handle Byzantine faulty agents) is log-linear in number of agents and linear in dimension of data points. The proposed algorithms differ from each other in the assumptions made for their correctness, and the gradient filter they use. \\
    
\end{abstract}

\section{Introduction}
\label{sec:intro}
This paper considers the problem of Byzantine fault tolerant distributed linear regression in a multi-agent system. The proposed algorithms, however, are applicable for a more general class of distributed optimization problems (described in Section~\ref{sec:pf}) that includes distributed linear regression. The system comprises of a server and $n$ agents, where each agent $i$ holds $n_i$ number of data points and responses, stacked as matrix $X_i \in \R^{n_i \times d}$ and vector $Y_i \in \R^{n_i}$, respectively. Up to $f$ of the $n$ agents in the system are Byzantine faulty and identity of Byzantine faulty agents is apriori unknown to the server~\cite{lamport1982byzantine, lynch1996distributed}. The server knows that if agent $i$ is honest (non-faulty) then its data points and responses satisfy $Y_i = X_i w^*$ for some unknown parameter value $w^* \in \R^d$. The objective of the server is to compute parameter $w^*$, regardless of the identity of Byzantine faulty agents. This seemingly simple problem is challenging to solve due to the adversarial nature of Byzantine faulty agents~\cite{bhatia2015robust}. In fact, it is well known that the existing techniques in robust statistical learning (cf.~\cite{huber2011robust}) are ineffective in solving the aforementioned problem unless certain assumptions on the probability distribution of agents' data points are satisfied~\cite{bhatia2015robust,chen2017distributed,blanchard2017machine}.\\

%for some known function\footnote{In general regression, $F$ is a guessed function and not known precisely. The optimization problem for solving general regression is known to be non-convex in $\H$ (set of honest agents is not known) and $w$. Therefore, it is conjectured that their does not exist an efficient method (a.k.a deterministic algorithm with polynomial computational complexity) for solving the general regression problem under Byzantine fault setting~\cite{su2016multi,fawzi2014secure,shoukry2017secure,su2016robust,bhatia2015robust}.}

Existing solutions for Byzantine fault tolerant distributed statistical learning (ref.~\cite{chen2017distributed, blanchard2017machine,damaskinos2018asynchronous,cao2018distributed,bernstein2018signsgd, alistarh2018byzantine, yin2018byzantine, data2018data}) rely on assumptions on the probability distribution of honest agents' data points for accuracy in probabilistic manner (even when their is no noise in the system). Whereas, we are interested in algorithms that can accurately (in absence of noise and with reasonably bounded error in presence of noise) compute $w^*$ in deterministic manner, under certain conditions on $f/n$, regardless of the probability distribution of agents' data points. We also note that all the prior works on Byzantine fault tolerance in distributed statistical learning assume synchronicity in the system, except~\cite{data2018data, damaskinos2018asynchronous} where every agent has access to all the data points and responses. Whereas, the proposed algorithms are \emph{partially asynchronous}, and therefore, robust to bounded delays in the system. \\

It should be noted that the above Byzantine fault tolerant linear regression can be used to solve a wide range of engineering problems pertaining to fault-tolerance or security, such as secure distributed state estimation of control systems~\cite{shoukry2017secure,fawzi2014secure,pajic2017attack, chong2015observability}, secure localization~\cite{li2005robust, zeng2013secure} and secure pattern recognition~\cite{wright2009robust}.

%rely on the necessary condition according to which the server can compute $w^*$ \emph{only if} $w^*$ is the unique parameter value that relates the data points and responses of any $n-2f$ honest agents~\cite{chong2015observability}. 

%The contribution of this paper can be summarized as follows.

\section{Summary of Contributions}
\label{sec:contri}
We propose two norm based filtering techniques, \emph{norm filtering} and \emph{norm-cap filtering}, that ``robustifies" the original distributed gradient descent algorithm to solve the aforementioned regression problem when $f/n$ is less than specified threshold values\footnote{Refer Section~\ref{sec:conv_anal} and Section~\ref{sec:conv_algo_2} for further details.}. The details of the algorithms are given in Sections~\ref{sec:algo} and~\ref{sec:algo_2}. The proposed algorithms also solve a more general multi-agent optimization problem where the honest agents' objective functions (or costs) satisfy certain assumptions, specified in Section~\ref{sec:pf}. The computational complexity of the proposed filtering techniques is $O(n (d + \log n))$, and the resultant algorithms are shown to be \emph{partially asynchronous}\footnote{Refer Section~\ref{sub:conv_2} for formal details.}. \\

Comparison of our paper with the existing related work is given in the following section.

% The server intends to learn the relationship between the data points and the responses, knowing that there exists one. The challenge is that some of the agents in the system are Byzantine-faulty whose identity is apriori unknown to the server~\cite{lamport1982byzantine, lynch1996distributed}. Therefore, the server can trust neither the data points nor the responses held by Byzantine-faulty agents. Moreover, the Byzantine faulty agents can collude together, and corrupt their data points and responses in an intelligent manner using the data points and responses of other honest (non-faulty) agents to mislead or manipulate the server~\cite{bhatia2015robust, blanchard2017machine}. Due to such strong adversarial behavior of Byzantine faulty agents, tolerance to Byzantine faulty agents imply tolerance to outliers and arbitrary corruptions to data points or responses (cf. robust statistics~\cite{huber2011robust}), but not the other way around~\cite{bhatia2015robust,blanchard2017machine,chen2017distributed}.\\

\section{Related Work}
Existing related work can be broadly classified into four categories:
\begin{enumerate}
    \item Regression with adversarial corruptions to data points or responses.
    \item Byzantine fault tolerant distributed estimation.
    \item Byzantine fault tolerant distributed learning.
    \item Byzantine fault tolerant distributed multi-agent optimization.
\end{enumerate}

\subsection{Regression with adversarial corruptions}
The aforementioned Byzantine fault-tolerant regression problem has been addressed for the centralized setting by many researchers in recent years (ref.~\cite{bhatia2015robust, mcwilliams2014fast, chen2013robust,  ren2019secure, prasad2018robust,diakonikolas2018sever}), where the server has access to all the agents' data points and responses. We are interested in a \emph{distributed} setting, where the data points and responses are distributed amongst agents, and are inaccessible to the server. 
%\footnote{In the centralized setting of this problem, the responses are often formulated as $Y = X w^* + B$, where $(X,\, Y)$ represents the data points and responses, and $B$ represents a sparse adversarial (Byzantine) corruption-vector~\cite{bhatia2015robust}.}
%It should be noted that Byzantine fault tolerance is not the same as robustness to arbitrary outliers as modeled in~\cite{prasad2018robust,diakonikolas2018sever}, mainly because Byzantine faulty agents can (actively) adapt to robustness techniques, run by the server, and mislead the server to a parameter value of their choice~\cite{bhatia2015robust, blanchard2017machine}. 
%In fact, the server could be completely unaware of the data points held by any agent in the system. 

\subsubsection{Challenges of distributed over centralized setting}
The challenges of distributed setting over the centralized counterpart are as follows.
\begin{enumerate}
\item Both the data points and responses of Byzantine faulty agents can be corrupted. Some of the centralized techniques (cf. \cite{bhatia2015robust, mcwilliams2014fast}) assume only corrupted responses. 

\item Agents could be holding large volume of data points and responses, that would make sharing of the entire data set with the server quite expensive in terms of the communication cost. Most of the centralized techniques (cf.~\cite{bhatia2015robust}) require the server to have access to all the agents' data points and responses. 

%\item Corruptions in the data points and responses of the server need not be static, i.e. Byzantine faulty agents can corrupt their data points (or responses) dynamically during the execution of a regression algorithm~\cite{su2018finite, blanchard2017machine}. None of the centralized techniques consider dynamic corruptions.

\item Server and the agents need not be synchronous. All the centralized techniques rely on synchronicity in the system~\cite{bhatia2015robust, chen2013robust, mcwilliams2014fast, prasad2018robust, diakonikolas2018sever}. 
\end{enumerate}

Unlike the centralized techniques, our proposed algorithms do not require agents to share their data points or responses with the server, and it is \emph{partially asynchronous}.
%Thus, in our problem setting, nothing can be assumed about the data points or responses of Byzantine-faulty agents.  
While spectral filters proposed in~\cite{prasad2018robust, diakonikolas2018sever} can be used in the distributed setting, they rely on singular value decomposition (SVD) of agents' costs' gradients (in each iteration) and therefore, are orders of magnitude more computationally complex than the proposed norm based filters. Also, unlike~\cite{prasad2018robust, diakonikolas2018sever}, we are interested in computing $w^*$ precisely (in absence of noise and within a reasonably bounded error in presence of noise) in a deterministic manner. \\

The `hard-thresholding' based robust regression technique in~\cite{bhatia2015robust}, even for the centralized setting, is effective only if the data points satisfy a certain condition. This condition holds with ``high probability" if the probability distribution of the data points is Gaussian with zero mean~\cite{bhatia2015robust}. It should be noted that the efficacy of our proposed algorithms does not depend on any assumptions on the probability distribution of agents' data points. Therefore, the proposed algorithms have a much wider applicability than the solutions proposed in~\cite{bhatia2015robust}, even for the centralized case.

\subsection{Byzantine fault tolerant distributed estimation}
In a closely related work, Su and Shahrampour~\cite{su2018finite} propose \emph{coordinate-wise trimmed mean} filtering for ``robustifying" the distributed gradient descent method in a peer-to-peer network. However, they do not provide an explicit bound on the number of Byzantine faulty agents that can be tolerated using their filtering technique. The convergence of their algorithm relies on a technical assumption (assumption 1 in~\cite{su2018finite}) that imposes additional constraints, than required by our proposed algorithms, on agents' data points. This point is reiterated by an example in Section~\ref{sec:num}. Resilient estimation technique proposed by~\cite{chen2018resilient} requires agents to commit (or share) their data points and responses to the server (or some central authority in their case), whereas we are interested in distributed setting where agents do not share their data points or responses with the server or any other agent in the system. In recent years, there has been a significant amount of work in Byzantine fault-tolerant state estimation (both distributed and centralized) of linear time-invariant (LTI) dynamical systems~\cite{mitra2018byzantineresilient,fawzi2014secure, shoukry2017secure,pajic2017attack, ren2019secure}. However, it should be noted that Byzantine fault-tolerant state estimation (aka. secure state estimation) of LTI dynamical systems is a special case of the considered regression problem (ref.~\cite{mitra2018byzantineresilient,fawzi2014secure, shoukry2017secure,pajic2017attack, ren2019secure}). We also note that our proposed algorithms are significantly (orders of magnitude) simpler than some of the secure state estimation algorithms~\cite{shoukry2017secure,pajic2017attack}, albeit can handle relatively less number of Byzantine faulty agents.

%Authors in~\cite{mitra2018byzantineresilient} address the problem of Byzantine fault tolerant  \emph{observer}, in peer-to-peer network setting, for linear time-invariant (LTI) dynamical systems, and so in general, their technique is not applicable to the considered distributed regression problem.

\def\krum{{\sf Krum}}

\subsection{Byzantine fault tolerant distributed statistical learning}
In recent years, significant amount of progress has been made on Byzantine faulty tolerant distributed \emph{statistical parameter learning}~\cite{bernstein2018signsgd, damaskinos2018asynchronous , blanchard2017machine, cao2018distributed, chen2017distributed, alistarh2018byzantine, data2018data, xie2018generalized}. In~\cite{blanchard2017machine, xie2018generalized, damaskinos2018asynchronous, cao2018distributed, data2018data, bernstein2018signsgd} the agents assume the role of workers in the parallelization of the (stochastic) gradient descent method and therefore, agents have access to all the data points. In~\cite{data2018data}, the authors propose a data encoding scheme for tolerating Byzantine faulty workers. Whereas,~\cite{blanchard2017machine, xie2018generalized, damaskinos2018asynchronous, cao2018distributed, bernstein2018signsgd} rely on filters to ``robustify" the original distributed stochastic gradient descent method. In~\cite{chen2017distributed, yin2018byzantine,alistarh2018byzantine}, the agents have distributed data points and responses, however it is assumed that all the agents choose their data points and responses following a common probability distribution. Thus, the filtering (or encoding) techniques proposed in these papers are not guaranteed to be effective for the considered problem setting where no assumptions are made on the probability distribution of agents' data points. Moreover, we are interested in deterministic regression algorithms that compute $w^*$ in a deterministic manner. We also note that the computational complexity for the server in our proposed filtering techniques (both \emph{norm filtering} and \emph{norm-cap filtering}) is $O(n (d + \log n))$, which is significantly less than the filtering techniques proposed in~\cite{blanchard2017machine, chen2017distributed}. 
%Unlike the above Byzantine fault tolerant distributed learning techniques, we are interested in a deterministic solution that converges to the true parameter value $w^*$ (in absence of noise) and not an approximation, under certain conditions.

%In \cite{cao2018distributed}, the server has access to a reliable set of data points. 
%Dong et al. \cite{yin2018byzantine} evaluates two filtering algorithm; median-based gradient descent and trimmed mean base gradient descent for Byzantine fault-tolerance in distributed statistical learning where the agents draw data points from a common set. For the convergence analysis of median-based gradient descent they assume the variance of the gradients to have bounded variance (over the entire distribution), and for the case of trimmed mean gradient technique they assume the gradients to be sub-exponential.
% ~\cite{alistarh2018byzantine} assume that individual gradients are close to average of gradients of all honest agents as data points are chosen from a common set from a common distribution. Compared to these existing works in Byzantine-fault tolerant statistical learning, we do not make any stochastic assumptions on the data points of the agents nor do we require the data points to be common for all the agents.
\subsection{Byzantine fault tolerant distributed multi-agent optimization}
Byzantine faulty tolerant distributed multi-agent optimization has also received considerable attention in recent years~\cite{su2016fault, su2016robust, sundaram2018distributed, su2016multi, fanitabasi2018review}. The objective in that case is to compute the point of minimum of the weighted average cost of the honest agents. If the agents' costs are scalar (i.e. $\R \to \R$) then the server can achieve this objective with weights of at least $n-2f$ honest agents bounded away from zero~\cite{su2016fault, sundaram2018distributed}. This result is extended in~\cite{su2016robust} for multivariate cost functions, where the proposed technique relies on the assumption that agents' costs are weighted linear combination of finite number of convex functions. In general, this assumption does not hold for the regression problem considered in this paper. Further, it is known that the weights can not be uniform when there are non-zero number of Byzantine faulty agents in the system \emph{if} the costs are not correlated~\cite{su2016multi,sundaram2018distributed,su2016fault}. Interestingly, the necessary correlation between honest agents' costs that would admit equal (positive) weights for all the honest agents in Byzantine distributed multi-agent optimization problem remains an open problem. In this paper, we present a sufficient correlation between honest agents' costs under which the weights associated with honest agents' costs are equal and positive. Specifically, if there exists a common point of minimum for all the honest agents' costs (refer Section~\ref{sec:pf}) then the minimizer of the average cost of honest agents can be computed in presence of \emph{limited} (limits specified in Section~\ref{sec:conv_anal} and~\ref{sec:conv_algo_2}) number of Byzantine faulty agents. Moreover, the proposed algorithms solve this multi-agent optimization problem efficiently, under the aforementioned sufficient correlation. \\

%Now, if the weights associated with the costs of honest agents are equal and positive then computing the point of minimum of the weighted average cost of honest agents solves the considered distributed regression problem.

Authors in~\cite{yang2017byrdie} extend the results of~\cite{su2016multi} for multivariate cost functions by assuming that the original optimization problem can be split into \emph{independent} scalar sub-problems with \emph{strictly} convex objective costs. This assumption is quite strong and in general, does hold for the considered regression problem setting. Authors in \cite{xu2018robust} solve the Byzantine fault-tolerant distributed optimization problem, assuming that each and every agents' cost is \emph{strongly} convex, which implies that every honest agent can \emph{locally} compute $w^*$ in context of the considered regression problem. This assumption is quite strong (it basically trivializes the considered regression problem), and is not required for the effectiveness of our proposed algorithms. 

\subsection{Norm Clipping in Machine Learning}
\label{sub:norm_clip}
We note that norm clipping (or filtering) of gradients has been proposed before for solving other un-related problems in machine learning, namely the gradient explosion problem in training of recurrent neural networks~\cite{pascanu2012understanding}, and the privacy preservation problem in distributed stochastic gradient descent based training of deep feed-forward neural networks~\cite{shokri2015privacy}. However, in these works the gradients are clipped based on a constant threshold value, that needs to be apriori determined carefully, whereas our filtering techniques rely on relative ranking of gradients' norms at each iteration and does not require computation of any additional threshold value. 

\section*{Paper Organization}
The rest of the paper is organized as follows. In Section~\ref{sec:not}, we introduce the notation used throughout the paper. Section~\ref{sec:pf} presents formal description of the problem addressed, along with the assumptions made in the paper. Section~\ref{sec:algo} presents the first filtering technique, referred as \emph{norm filtering}. Section~\ref{sec:conv_anal} presents the convergence analysis of the resultant gradient descent algorithm with \emph{norm filtering}. Section~\ref{sec:algo_2} presents the second filtering technique, referred as \emph{norm-cap filtering}. Section~\ref{sec:conv_algo_2} presents the convergence analysis of the resultant gradient descent algorithm with \emph{norm-cap filtering}. Section~\ref{sec:num} presents a numerical example for demonstrating the obtained convergence results for the proposed algorithm. Finally, concluding remarks are made in Section~\ref{sec:conclude}. Appendix~\ref{app:noise} discusses the effect of system noise. Appendix~\ref{app:proofs} contains formal proofs of the results.

\section{Notations}
\label{sec:not}
$\Z$, $\N$, $\R$ and $\R^d$ denote sets of integers, natural numbers, real numbers and $d$-dimensional real-valued vectors, respectively. $\Z_{\geq 0}$, $\R_{\geq 0}$ and $\R_{> 0}$ represent non-negative integers, non-negative reals and positive reals, respectively. Let $[n] = \{1,\ldots,\,n\}$. For a vector $v \in \R^d$, $v[k]$ denotes its $k$-th element, and $\norm{v}$ denotes its Euclidean norm (or $2$-norm), which is equal to $\sqrt{\sum_{k} (v[k])^2}$. Notation $[a, b]^d$ for $a \leq b \in \R$ denotes a set of $d$-dimensional vectors with each element belonging to the interval $[a,b]$. For a matrix $M \in \R^{n \times d}$, $M^T$ denotes its transpose and $M[k] \in \R^{d}$ denotes a column vector corresponding its $k$-th row. In other words, $M[k]$ is the $k$-th column of $M^T$. For a set of matrices $\{M_i\}_{i \in S} = \{M_i \,| \, M_i\in \R^{n_i \times d}, \, i \in S\}$, the notation $[M_i]_{i \in S}$ represents the row-wise concatenation of the matrices $\{M_i\}_{i \in S}$ (stacking of the matrices). Thus, $[M_i]_{i \in S}$ is a matrix of dimensions $(\sum_{i \in S}n_i) \times d$. Inner product (or scalar product) of two vectors $v_1,\, v_2$ in $\R^d$ is denoted by $\iprod{v_1}{v_2}$ and is equal to $v^T_1 v_2$. For a multivariate differentiable function $C: \R^d \to\R$, $\nabla C (v)$ denotes is gradient at a point $v \in R^d$. For a finite set $S \subset \Z$, $\mnorm{S}$ denotes its cardinality. For real number $x \in \R$, $\mnorm{x}$ denotes its absolute value.  

\section{Optimization Framework}
\label{sec:pf}
As mentioned earlier, we consider a system of $n$ agents and a server, with communication links between all the agents and the server. Agents do not communicate with each other. The system contains at most $f$ Byzantine faulty agents that can behave arbitrarily~\cite{lynch1996distributed, lamport1982byzantine}. The identity of Byzantine faulty agents is apriori unknown to the server. However, the server knows the value of $f$. Let $\H$ and $\B$ denote the sets of honest (non-faulty) agents and Byzantine faulty agents, respectively. \\

In this paper, we propose an algorithm to solve a distributed multi-agent optimization problem where 
%Any information communicated by Byzantine faulty agents can not trusted by the server and  Byzantine faulty agents can collude~\cite{lynch1996distributed, bhatia2015robust}. In precise terms, the server can not make any assumptions about the behavior of Byzantine faulty agents. 
% \textbf{Regression Model:} The responses $Y_i$ and the data points $X_i$ are related as $Y_i[k] = F(X_i[k], w^*), \, \forall k \in [n_i]$ for all $i \in \H$, where $F: \R^d \times \R^d \to \R$ is a differentiable function \emph{known} to the server\footnote{In general regression, $F$ is a guessed function and not known precisely. The optimization problem for solving general regression is known to be non-convex in $\H$ (set of honest agents is not known) and $w$. Therefore, it is conjectured that their does not exist an efficient method (a.k.a deterministic algorithm with polynomial computational complexity) for solving the general regression problem under Byzantine fault setting~\cite{su2016multi,fawzi2014secure,shoukry2017secure,su2016robust,bhatia2015robust}.} and $w^* \in \R^d$ is an \emph{unknown} parameter. The function $F$ and $w^*$ are common for all the honest agents. The objective of the server is to learn a $w^*$ in $\R^d$ that relates the data points and the responses of all the honest agents in the aforementioned manner.
each agent $i \in \H$ is associated with a differentiable \emph{convex} cost $C_i(w): \R^d \to\R$, that satisfies certain assumptions that are mentioned below. The objective of the server is to compute a point of minimum of the average cost of the honest agents, 
\begin{align}
    C_{\H}(w) = \frac{1}{|\H|}\sum_{i \in \H}C_i(w), \quad \forall w \in \R^d \label{eqn:obj}
\end{align}
In Section~\ref{sub:lin_pf}, we demonstrate the applicability of this optimization framework for the case of least squared-error distributed linear regression. In this optimization problem, we assume the following:
% It is ensured that  
% $C_i(w) \geq 0, \, \forall w \in \R^d$ and $C_i(w^*) = 0$ for all $i \in \H$. 
% As $C_{\H}(w) \geq 0, \, \forall w \in \R^d$ and $C_{\H}(w^*) = 0$, thus $w^*$ is a point of minimum of $C_{\H}(w)$. Hence, solving the aforementioned regression problem reduces to computing the point of minimum of $C_{\H}$.
%We reiterate that this optimization objective is not trivial as the identity of $\H$ and $\B$ is prior unknown to the server, and Byzantine faulty agents can behave arbitrarily. 

\begin{itemize}
    \item[\textbf{(A1)}] \textbf{Unique point of minimum and strong convexity of reduced average cost:} \\
Assume that $C_{\H}$ has a unique point of minimum $w^*$ in a compact and convex set $\W \subset \R^d$. Further, for any $\hat{\H} \subseteq \H$ of cardinality at least $n-f$, assume that the average cost of $\hat{\H}$, i.e. $C_{\hat{\H}} = (1/\mnorm{\hat{\H}})\sum_{i \in \hat{\H}}C_i$, is strongly convex. Specifically, 
\[\iprod{w-w'}{\nabla C_{\hat{\H}}(w) - \nabla C_{\hat{\H}}(w')} \geq \lambda \norm{w - w'}^2,\, \forall w, \, w' \in \R^d\]
where $\lambda \in \R_{> 0}$. 
% This means that $[Y_i]_{i \in \H} = ([X_i]_{i \in \H}) w$ in $\W$ if and only if $w = w^*$.

\item[\textbf{(A2)}] \textbf{$\{C_i\}_{i \in \H}$ minimizes at $w^*$ and $\{\nabla C_i\}_{i \in \H}$ are Lipschitz continuous:} \\
For every $i \in \H$, assume that $C_i(w) \geq C_i(w^*), \, \forall w \in \R^d$, and \[\norm{\nabla C_{i}(w) - \nabla C_{i}(w')} \leq \mu \norm{w - w'}, \, \forall w, \,w' \in \R^d,\]
where $\mu \in \R_{\geq 0}$. 

\item[\textbf{(A3)}] \textbf{Strength of Byzantine faulty agents is less than majority:} \\
Assume that the maximum number of Byzantine faulty agents is less than the half of the total number of agents, i.e. 
\[f < n/2\]  
It should be noted that it is impossible to compute $w^*$ if $f \geq n/2$ in general when no assumptions are made on the probability distribution of honest agents' data points~\cite{bhatia2015robust,fawzi2014secure,shoukry2017secure}.
\end{itemize}

\subsection{Least Squared-Error Distributed Linear Regression}
\label{sub:lin_pf}
Now, consider the distributed linear regression problem where each agent $i \in [n]$ is associated with $n_i$ number of data points and responses, represented by $X_i \in \R^{n_i \times d}$ and $Y_i \in \R^{n_i}$, respectively. The server knows that for each agent $i \in \H$, $Y_i = X_i w^*$ for some parameter $w^* \in \R^d$. The parameter $w^*$ is unknown to the server and is common for all the honest agents (cf.~\cite{bhatia2015robust}). The objective of the server is to learn a value of $w^*$ (need not be unique). To solve this regression problem, each agent $i \in \H$ defines the following squared-error cost 
\[C_i(w) = \frac{1}{2}\norm{Y_i - X_i w}^2 = \frac{1}{2}\left(w^T X_i^T X_i w - 2 X_i^T Y_i w + \norm{Y_i}^2\right), \, \forall w \in \R^d, \, \forall i \in \H\] 
As $v^T X_i^T X_i v = \norm{X_i v}^2, \, \forall v \in \R^d$, thus $X_i^T X_i$ is a positive semi-definite matrix. Thus, $C_i$ is convex for all $i \in \H$. Here,
\[\nabla C_i(w) = X_i^T (X_i w - Y_i), \, \forall w \in \R^d, \, \forall i \in \H\]
As $Y_i = X_i w^*, \, \forall i \in \H$, thus $\nabla C_i(w^*) = 0,\, \forall i \in \H$. As the costs $\{C_i\}_{i \in \H}$ are convex, this implies that $w^*$ is a point of minimum for all $\{C_i\}_{i \in \H}$. 
% Moreover, 
% \[\norm{\nabla C_i(w) - \nabla C_i(w')} = \norm{X_i^T X_i (w - w')} = \sqrt{(w-w')^T (X_i^TX_i)^2 (w-w')} \]
As $X_i^T X_i$ is positive semi-definite, therefore (cf.~\cite{horn1990matrix}) 
\[0 \leq v^T \left(X_i^T X_i\right)^2 v \leq \overline{\nu}_i^2\norm{v}^2, \, \forall v \in \R^d, \, \forall i \in \H\] 
where $\overline{\nu}_i$ is the largest eigenvalue of $X_i^TX_i$. This implies,
\[\norm{\nabla C_i(w) - \nabla C_i(w')} = \norm{X_i^T X_i (w - w')} = \sqrt{(w-w')^T (X_i^TX_i)^2 (w-w')} \leq \overline{\nu}_i \norm{w - w'}\]
for all  $w, \, w' \in \R^d$. Thus, for $\mu = \max_{i \in \H}\overline{\nu}_i \geq 0$, we get
\[\norm{\nabla C_i(w) - \nabla C_i(w')} \leq \mu \norm{w - w'}, \, \forall w, \, w' \in \R^d, \, \forall i \in \H\]
Hence, assumption~\textbf{(A2)} holds naturally for the case of least squared-error linear regression. For any set $\hat{\H} \subseteq \H$, the average cost $C_{\hat{\H}}$ is 
\[C_{\hat{\H}}(w) = \frac{1}{\mnorm{\hat\H}}\sum_{i \in \hat\H} C_i = \frac{1}{2\mnorm{\hat\H}}\sum_{i \in \hat\H}\norm{Y_i - X_i w}^2 = \frac{1}{2\mnorm{\hat\H}} \norm{Y_{\hat{\H}} - X_{\hat\H} w}^2, \, \forall w \in \R^d\]
where, $Y_{\hat\H} = [Y_i]_{i \in \hat\H}$ and $X_{\hat\H} = [X_i]_{i \in \hat{\H}}$ are the stacked responses and data points of all the agents in $\hat\H$. Thus, 
\[\nabla C_{\hat\H}(w) = \frac{1}{\mnorm{\hat\H}} X_{\hat\H}^T (X_{\hat\H} w - Y_{\hat\H}), \, \forall w \in \R^d\]
Therefore, 
\[\iprod{w-w'}{\nabla C_{\hat\H}(w) - \nabla C_{\hat\H}(w')} = \frac{1}{\mnorm{\hat\H}} (w - w')^T X_{\hat\H}^T X_{\hat\H} (w - w') \geq \frac{\underline{\nu}_{\hat\H}}{\mnorm{\hat\H}} \norm{w - w'}^2, \, \forall w, \, w' \in \R^d\]
where, $\underline{\nu}_{\hat\H}$ is the smallest eigenvalue of $X_{\hat\H}^TX_{\hat\H}$. Thus, if the stacked matrix $[X_i]_{i \in \hat{\H}}$ has rank equal to $d$, i.e. $w^*$ can be uniquely computed from the responses and data points of honest agents in $\hat\H$, then not only $w^*$ is the unique point of minimum of $C_{\hat{\H}}(w)$, but $C_{\hat\H}$ is also strongly convex as $\underline{\nu}_{\hat\H} > 0$ (cf.~\cite{horn1990matrix}). In other words, if $w^*$ can be uniquely determined given the data points and responses of agents in $\hat\H$, for all $\hat\H \subseteq \H$ of cardinality $n-f$ then assumption~\textbf{(A1)} holds, and
\[\lambda = \frac{1}{\mnorm{\hat\H}}\left(\min_{\hat\H \subseteq \H, \, \mnorm{\hat\H} = n-f} \underline{\nu}_{\hat\H} \right) > 0\]
In the discussion above, we only consider the noiseless case. However, the proposed algorithms are effective even when there is (bounded) noise in the system, as discussed in Appendix~\ref{app:noise}. 
\section{Algorithm-I: Gradient Descent with Norm Filtering}
\label{sec:algo}

The algorithm follows the philosophy of gradient descent based optimization. The server starts with an arbitrary estimate of the parameter and updates it iteratively in two simple steps. In the first step, the server collects gradients of all the agents' costs (at the current estimated value of the parameter) and sort them in the increasing order of their $2$-norms (breaking ties arbitrarily in the order). In the second step, the server filters out the gradients with $f$ largest $2$-norms, and uses the (vector) sum of the remaining gradients as update direction. Therefore, the filtering scheme is referred as \emph{norm filtering}. The algorithm is formally described as follows. \\

Server begins with an arbitrary estimate $w^0 \in \W$ of the parameter $w^*$ and iteratively updates it using the following steps. We let $w^t$ denote the parameter estimate at time $t \in \Z_{\geq 0}$.

\begin{enumerate}
    \item[S1:] At each time $t \in \Z_{\geq 0}$, the server requests from each agent the gradient of its cost at the current estimate $w^{t}$, and sorts the received gradients by their norms. Let,
\[\norm{g^t_{i_1}} \leq \ldots \leq \norm{ g^t_{i_{n-f}}} \leq \ldots \leq \norm{g^t_{i_{n}}}\]
where, $i_k \in [n], \, \forall k \in [n]$ and $g_i^t$ denotes the gradient reported by agent $i$ at time $t$. Note that if $i \in \B$ then $g^t_i = \star$ (arbitrary), and if $i \in \H$ and the system is synchronous then $g^t_i = \nabla C_i(w^t)$ (asynchronous case is discussed in Section~\ref{sub:conv_2}). Let, 
\begin{align}
    \F_t = \{i_1,\ldots, i_{n-f}\} \label{eqn:filter}
\end{align}
be the set of agents with $n-f$ smallest gradient norms at time $t$. 
    \item[S2:] The server updates $w^t$ as,
\begin{align}
    w^{t+1} = \left[ w^t - \eta_{t} \cdot \sum_{\sigma \in \F_t} g^t_{\sigma} \right]_{\W}, \, \forall t \in \Z_{\geq 0} \label{eqn:algo_1}
\end{align}
where, $\{\eta_{t}\}$ is a sequence of bounded positive real values and $[\,\cdot \,]_\W$ denotes projection onto $\W$ w.r.t. Euclidean norm, i.e. $[w]_{\W} = \arg \min_{v \in \W} \norm{w - v}, \, \forall w \in \R^d$.
% , and
% \begin{align}
%     \overline{g^t_{\sigma}} = \left\{ \begin{array}{ccc} g^t_{\sigma} & , & \norm{g^t_\sigma} \leq \norm{g^t_{i_{n-f}}} \\ 
%     0 &, & \text{otherwise} \end{array}  \right. \label{eqn:filt}
% \end{align}
\end{enumerate}

\subsection{Computational Complexity}
\label{sub:comp}
In Step S1, the server computes the norm of all reported gradients in $O(nd)$ time. Sorting of these norms takes additional $O(n\log n)$ time. Thus, the net computational complexity of norm filtering (for the server) is $O(n (d + \log n))$. Whereas, computational complexity of each agent $i \in \H$ is $O(n_i d)$.\\

In Step S2, the server adds all the vectors in set $\F_t$ to update its parameter estimate in $O(nd)$ time. The projection of the updated estimate on a known compact convex set $\W$, defined using affine constraints (a bounded polygon), can be done in $O(d^{3})$ time using quadratic programming algorithm in~\cite{ye1989extension}.
Therefore, the net computational complexity of the algorithm (for the server) is $O(n (d + \log n) + d^3)$ per iteration.

\subsection{Intuition}
The principal factor behind the convergence of the proposed algorithm is consensus amongst all the honest agents on $w^*$. Norm filtering bounds the norms of all the gradients used for computing the update direction (even if they are Byzantine faulty gradients) by norm of an honest agent's gradient (as there could be at most $f$ Byzantine faulty agents). This has two-fold implications,
\begin{enumerate}
    \item As the gradients of all the honest agents' costs vanish at $w^*$ (cf. assumption~\textbf{(A2)} and Claim~\ref{clm:red}), therefore $w^*$ is ensured to be a fixed-point of the iterative algorithm~\eqref{eqn:algo_1}.
    \item As gradients of all the honest agents' costs are Lipschitz continuous (assumption~\textbf{(A2)}), therefore the magnitude of the contribution of the adversarial gradients (reported by Byzantine faulty agents) in the update direction is bounded above by the separation between current estimate $w^t$ and $w^*$ (cf. Claim~\ref{clm:red}).
\end{enumerate}
The proposed filtering allows contribution of at least $n-2f$ honest agents' gradients ($f<n/2$ by assumption~\textbf{(A3)}), that pushes the current estimate $w^t$ towards $w^*$ with force that is also proportional to the separation between current estimate $w^t$ and $w^*$ for small enough $f/n$, due to the strong convexity assumption~\textbf{(A1)}. This gives us an intuition that effect of adversarial gradients can be overpowered by the honest agents' gradients in Step S2 at all times if $f/n$ is small enough. \\

The insight above is conducive to the formal convergence results presented in the next section, for both synchronous (Section~\ref{sub:sync}) and asynchronous (Section~\ref{sub:conv_2}) cases.

\section{Convergence Analysis: Algorithm-I}
\label{sec:conv_anal}

Before we present the convergence results for Algorithm-I, let us note the following implications of assumptions~\textbf{(A1)} and~\textbf{(A2)}.
\begin{claim}
\label{clm:red}
Assumptions \textbf{(A1)}-\textbf{(A2)} imply that
\begin{align}
    \mu \geq \lambda. \label{eqn:imp_1}
\end{align}
%\nabla C_i (w^*) = 0 , \, \forall i \in \H & \text{, and }
Moreover, if $f/n < 1 / (1 + (\mu/\lambda))$ then for any $\H' \subset \H$ of cardinality $|\H'| = n - 2f$, we get
\begin{align}
    \nabla C_{\H'} (w) = 0 \text{ in } \W \text{ iff } w = w^* \label{eqn:restrict}
\end{align}
where, $C_{\H'} = (1/|\H'|)\sum_{i \in \H'}C_i$.
\end{claim}
\begin{proof}
Refer to Appendix~\ref{sub:red}.
\end{proof}

We rely on the following sufficient criterion for the convergence of non-negative sequences.
\begin{lemma}[Ref. Bottou, 1998 \cite{bottou1998online}]
\label{lem:seq_conv}
Consider a sequence of real values $\{u_t\}, \, t \in \Z_{\geq 0}$. If $u_t \geq 0, \,\forall t \in \Z_{\geq 0}$ then 
\begin{align}
    \sum_{t = 0}^\infty (u_{t+1} - u_t)_{+} = S^{+}_{\infty} < \infty \implies \left\{\begin{array}{c} u_t \underset{t \to \infty}{\longrightarrow} u_\infty < \infty \\ \\ \sum_{t = 0}^\infty (u_{t+1} - u_t)_{-} = S^{-}_{\infty} > -\infty \end{array}\right.
\end{align}
where the operators $(\cdot)_{+}$ and $(\cdot)_{-}$ are defined as follows ($x \in \R$),
\begin{align*}
    (x)_{+} = \left\{\begin{array}{ccc} x &, & x > 0\\ 0 &, & \text{otherwise} \end{array}\right. \text{, and } (x)_{-} = \left\{\begin{array}{ccc} 0 &, & x > 0\\ x &, & \text{otherwise} \end{array}\right.
\end{align*}
In other words, convergence of infinite sum of positive variations of a non-negative sequence is sufficient for the convergence of the sequence and infinite sum of its negative variations. 
\end{lemma}

\subsection{Convergence With Full Synchronism}
\label{sub:sync}
We now present the sufficient conditions under which the proposed algorithm converges to $w^*$ when the server and honest agents are synchronous, i.e. we assume: \\

\noindent \textbf{(A4)} \textbf{Full Synchronism:} $g^t_i = \nabla C_i(w^t), \, \forall i \in \H$ for all $t \in \Z_{\geq 0}$.

\begin{theorem}
\label{thm:mr_1}
Under assumptions \textbf{(A1)}-\textbf{(A4)}, if $\sum_{t = 0}^{\infty} \eta_t = \infty$, $\sum_{t = 0}^{\infty} \eta^2_t < \infty$, and
\begin{align}
    \frac{f}{n} < \frac{1}{1 + 2(\mu/\lambda)} \label{eqn:cond_1}
\end{align}
% \begin{align}
%     \frac{f}{n} < \frac{1}{3(\mu/\lambda)} \label{eqn:cond_1}
% \end{align}
%$\beta$ in \textbf{(A3)} is bounded below as
%\[\beta \geq \frac{1 + \Pi \cos \theta}{\Pi \cos \theta} \text{ and } n > \beta f \]
then the sequence of parameter estimates $\{w^t\}$, generated by~\eqref{eqn:algo_1}, converges to $w^*$.
\end{theorem}
\begin{proof}
Refer Appendix~\ref{sub:mr_1}.
\end{proof}

%It turns out that strong convexity of the average cost of honest agents (assumption \textbf{(A1)}) and Lipschitz continuity of the gradients of the honest agents' costs (assumption \textbf{(A2)}) plays a crucial role in resilience of the proposed gradient descent algorithm with norm filtering against Byzantine faulty agents.\\

Theorem \ref{thm:mr_1} states that if $f/n$ is less than $1/(1 + 2(\mu/\lambda))$ then the proposed algorithm will reach the point of minimum of the $C_{\H}$ asymptotically under assumptions \textbf{(A1)}-\textbf{(A4)}. As assumptions \textbf{(A1)}-\textbf{(A3)} also imply that $\mu \geq \lambda$ (cf.~Claim \ref{clm:red}), thus $f$ (maximum allowable Byzantine agents) should be less than one-third of $n$ (total number of agents) for the proposed algorithm to converge to $w^*$.\\

If assumptions \textbf{(A1)}-\textbf{(A2)} and condition~\eqref{eqn:cond_1} are satisfied, then 
\[f/n < 1/(1 + 2(\mu/\lambda)) < 1/(1 + (\mu/\lambda))\] 
and thus (cf. Claim~\ref{clm:red}), 
\begin{align*}
    \nabla C_{\H'}(w) = 0 \text{ in } \W \text{ iff } w = w^*
\end{align*}
for all $\H' \subset \H$ subject to $\mnorm{\H'} = n-2f$. In other words, the point of minimum of the average cost of any $n-2f$ honest agents is the point of minimum of the average cost of all honest agents. Therefore, under condition~\eqref{eqn:cond_1} and assumptions~\textbf{(A1)}-\textbf{(A2)}, $C_{\H'}$ is indeed strongly convex for all $\H' \subset \H$ of cardinality $n-2f$.\\

It is known, from control systems literature~\cite{fawzi2014secure, pajic2014robustness, shoukry2017secure, chong2015observability}, that the considered linear regression problem can be solved in presence of at most $f$ Byzantine faulty agents \emph{only if} matrix 
\[X_{\H'} = [X_i]_{i \in \H'} \in \R^{(\sum_{i \in \H'}n_i)\times d}\]
has rank equal to $d$ for every subset $\H' \subset \H$ of cardinality $n-2f$. In light of this information, we make the following additional assumption on the costs $\{C_i\}_{i \in \H}$ to improve the tolerance bound on $f/n$.

\begin{itemize}
    \item[\textbf{(A5)}] \textbf{Uniform $f$-Redundancy:} \\
    For any $\H' \subset \H$ of cardinality $n-2f$, we assume that
\[\iprod{w - w'}{\nabla C_{\H'}(w) - \nabla C_{\H'}(w') } \geq \gamma \norm{w-w'}^2, \ \forall w, \, w' \in \R^d\]
where, $C_{\H'}(w) = (1/\mnorm{\H'}) \sum_{i \in \H'}C_{i}(w)$ and $\gamma \in \R_{> 0}$.
\end{itemize}

For the case of least squared-error linear regression (refer Section~\ref{sub:lin_pf}), similar to $\lambda$ in assumption~\textbf{(A1)}, we have
\[\gamma = \frac{1}{\mnorm{\H'}}\left(\min_{\H' \subset \H, \, \mnorm{\H'} = n-2f} \underline{\nu}_{\H'} \right)\]
where, $\underline{\nu}_{\H'}$ is the smallest eigenvalue of $X_{\H'}^T X_{\H'}$. We refer the above redundancy as \emph{uniform} because it is required to hold for all $\H' \subset \H$ of cardinality $n-2f$. This $f$-redundancy property of the regression problem is also referred as $2f$-\emph{sparse observability} in control systems literature~\cite{chong2015observability}. Also, note that assumption~\textbf{(A5)} is meaningful only if assumption~\textbf{(A3)} holds, i.e. $f < n/2$. \\

Similar to Claim~\ref{clm:red}, 
\begin{claim}
\label{clm:mu_gamma}
Assumptions \textbf{(A2)}-\textbf{(A3)} and~\textbf{(A5)} imply that $\mu \geq \gamma$
\end{claim}
\begin{proof}
Refer Appendix~\ref{sub:mu_gamma}
\end{proof}

With assumption~\textbf{(A5)}, we get the following alternate convergence result for the proposed algorithm.
\begin{theorem}
\label{thm:mr_2}
Under assumptions \textbf{(A1)}-\textbf{(A5)}, if $\sum_{t = 0}^{\infty} \eta_t = \infty$, $\sum_{t = 0}^{\infty} \eta^2_t < \infty$, and
\begin{align}
    \frac{f}{n} < \frac{1}{2 + \mu/\gamma} \label{eqn:cond_2}
\end{align}
then the sequence of parameter estimates $\{w^t\}$, generated by~\eqref{eqn:algo_1}, converges to $w^*$.
\end{theorem}
\begin{proof}
Refer Appendix~\ref{sub:mr_2}. 
\end{proof}

Theorem \ref{thm:mr_2} states that if $f/n$ is less than $1/(2 + \mu/\gamma)$ then the proposed algorithm reaches the point of minimum of the $C_{\H}$ asymptotically under assumptions \textbf{(A1)}-\textbf{(A5)}.  Owing to Claim~\ref{clm:mu_gamma}, the right-hand side in condition~\eqref{eqn:cond_2} is less than or equal to $1/3$.\\

Instead of using a diminishing step-size, we can use a \emph{small enough} constant step-size in~\eqref{eqn:algo_1} to obtain linear convergence of the proposed algorithm as stated below.

\begin{theorem}
\label{thm:exp_conv}
Under assumptions \textbf{(A1)}-\textbf{(A5)}, if condition~\eqref{eqn:cond_2} is satisfied then for 
\[\eta_t = \eta = \frac{n\gamma - f(2 \gamma + \mu)}{\mu^2 (n-f)^2} > 0, \quad \forall t \in \Z_{\geq 0},\]
the sequence of parameter estimates $\{w^t\}$, generated by~\eqref{eqn:algo_1}, converges linearly to $w^*$, with
\begin{align*}
    \norm{w^{t+1} - w^*} \leq \rho \norm{w^t - w^*}, \quad \forall t \in \Z_{\geq 0}
\end{align*}
where $\rho = \sqrt{1 - 2\eta(n\gamma - f(2 \gamma + \mu)) + \mu^2 (n-f)^2 \eta^2}$ is a positive real number of value less than $1$.
\end{theorem}
\begin{proof}
Refer Appendix~\ref{sub:exp_conv}. 
\end{proof}
%Note that $\mu$ can be ensured to be greater than unity simply by scaling the collective data points and responses of honest agents, without affecting the solution of the regression problem or the validity of conditions~\eqref{eqn:cond_1}-\eqref{eqn:cond_2} .

\subsection{Convergence With Partial Asynchronism}
\label{sub:conv_2}
In practice, the server and the agents need not synchronize. At any given time $t$, some of the honest agents might not be able to report gradients of their costs at the current estimate $w^t$. This could occur due to various reasons, such as hardware malfunction or large communication delays. In order to cope with such irregularities, the server uses the last reported gradient, in step S2, of an agent that fails to report its cost's gradient at the current estimate in step S1. Formally, for an agent $i \in [n]$ that fails to report its gradient at $t$, the server uses the last reported gradient $g_i^{t - s_i(t)}$ of that agent, where $s_i(t) \in \Z_{\geq 0}$ is the time passed since agent $i$ reported its gradient. However, we assume $s_i(t)$ to be bounded for all $i \in \H$. In other words, we assume \emph{partial asynchronism} that is formally stated as follows (cf. Section 7.1 of Bertsekas and Tsitsiklis, 1998 ~\cite{bertsekas1989parallel}).

\begin{itemize}
\item[\textbf{(A6)}] \textbf{Partial Asynchronism:} \\
For every $i \in \H$, $g_i^{t} = \nabla C_i(w^{t - s_i(t)}), \, \forall t \in \Z_{\geq 0}$ where $0 \leq s_i(t) \leq t_o$. \\
Here, $t_o$ is a finite (unknown) positive integer. As the server uses the last available gradient at each time $t$ for each agent $i \in [n]$, thus $s_i(t+1) \leq 1 + s_i(t), \, \forall t \in \Z_{\geq 0}, \, \forall i \in \H$. \\
If the server does not receive any gradient from an agent $i \in [n]$ until time $t$ (i.e. $t - s_i(t) < 0$), then it assigns $g_i^{t} = 0$. 
%Alternately, if $t - s_i(t) < 0$ for any $i \in [n]$, this implies the server has not received any gradient from agent $i$ up to time $t$.
\end{itemize}
%\footnote{This implies that $t - s_i(t) \leq t+1 - s_i(t+1),  \, \forall t \in \Z_{\geq 0}, \, \forall i \in \H$.}

If $t_o = 0$ then assumption~\textbf{(A6)} is equivalent to assumption~\textbf{(A4)}, for which case the sufficient conditions for convergence of $\{w^t\}$ to $w^*$ have already been stated in Theorems~\ref{thm:mr_1},~\ref{thm:mr_2} and~\ref{thm:exp_conv}. Therefore, in assumption~\textbf{(A6)} $t_o > 0$. Before we state the result on the convergence result under~\textbf{(A6)}, let us first establish that the infinite sum of the sequence $\left\{ \eta_t \norm{w^{t} - w^{t - s_i(t)}}\right\}$ for all $i \in \H$ is finite ($< \infty$). This result is used later for showing convergence of $\{w^t\}$, generated by~\eqref{eqn:algo_1}, to $w^*$ under the aforementioned partial asynchronism.
%Note that, just like all the assumptions before, we are not making any assumptions about the gradients reported by the Byzantine faulty agents in assumption~\textbf{(A6)}.\\

\begin{lemma}
\label{lem:bnd_growth}
Consider the update law~\eqref{eqn:algo_1} under assumptions~\textbf{(A1)}-\textbf{(A3)} and \textbf{(A6)}. If $\eta_{t+1} \leq \eta_t, \, \forall t \in \Z_{\geq 0}$ and $\sum_{t = 0}^{\infty}\eta^2_t < \infty$ then 
\begin{align*}
    \sum_{t = 0}^\infty \eta_t\norm{w^{t} - w^{t - s_i(t)}} < \infty, \, \forall i \in \H
\end{align*}
\end{lemma}

\begin{proof}
Refer Appendix~\ref{sub:bnd_growth}.
\end{proof}

The result in Lemma~\ref{lem:bnd_growth} does not require the sequence $\{\eta_t\}$ to be monotonically decreasing as long as $\sum_{t = 0}^\infty \eta^{2}_t < \infty$. However, the proof is simplified under this assumption and a non-monotonous $\eta_t$ does not confer any additional advantages as far as \emph{asymptotic} convergence of $\{w^t\}$ is concerned. Also, the commonly used diminishing step-size $\eta_t = 1/(t+1), \, \forall t \in \Z_{\geq 0}$ is indeed monotonically decreasing (cf.~\cite{rudin1964principles}). \\

%Now, using similar deductions as made for Theorem~\ref{thm:mr_2}, we obtain the following convergence result for the case of partially asynchronous system.
%We only mention the honest agents in the above assumption as we can not assume anything definite about the behavior of Byzantine faulty agents.

\begin{theorem}
\label{thm:mr_delay}
Under assumptions \textbf{(A1)}-\textbf{(A3)}, \textbf{(A5)} and \textbf{(A6)}, if $\eta_{t+1} \leq \eta_t, \, \forall t \in \Z_{\geq 0}$, $\sum_{t = 0}^{\infty} \eta_t = \infty$, $\sum_{t = 0}^{\infty} \eta^2_t < \infty$, and condition~\eqref{eqn:cond_2} holds then the sequence of parameter estimates $\{w^t\}$, generated by~\eqref{eqn:algo_1}, converges to $w^*$.
\end{theorem}

\begin{proof}
Refer Appendix~\ref{sub:mr_delay}.
\end{proof}

The convergence result stated in Theorem~\ref{thm:mr_delay} is same as that in Theorem~\ref{thm:mr_2}, if the partial asynchronicity assumption (i.e.~\textbf{(A6)}) is replaced by the synchronicity assumption (i.e.~\textbf{(A4)}). Similarly, the convergence result stated in Theorem~\ref{thm:mr_1} is also valid if assumption~\textbf{(A4)} (full synchronism) in Theorem~\ref{thm:mr_1} is replaced by assumption~\textbf{(A6)} (partial asynchronism).

% \subsection{\emph{Generalized} Byzantine Faults}
% The above convergence results (Theorems~\ref{thm:mr_1},~\ref{thm:mr_2},~\ref{thm:exp_conv} and ~\ref{thm:mr_delay}) hold verbatim even when the agents switch roles, i.e. $\H$ and $\B$ change with time $t$, as long as there are at most $f$ Byzantine faulty agents and assumptions~\textbf{(A1)}-\textbf{(A3)} hold at any time $t$. This Byzantine faults model is also referred as \emph{generalized} Byzantine faults~\cite{xie2018generalized}.

%%%%%%%%%%%%%%%%%%%%%%%%%%%%%%NUMERICAL EXAMPLE %%%%%%%%%%%%%%%%%%%%%%%%%%%%%%%%
\section{Algorithm-II: Gradient Descent With Norm-Cap Filtering}
\label{sec:algo_2}
The algorithm in essence is similar to Algorithm-I, only here instead of eliminating the $f$ largest agents' gradients the server caps the  $f$ largest gradients' norms by the norm of  $(f+1)$-th largest reported gradient. Therefore, the filtering scheme is referred as \emph{norm-cap filtering}. Expectedly, norm-cap filtering improves the sufficiency bound on $f/n$ with respect to~\eqref{eqn:cond_2}. The steps of the algorithm are formally described as follows. \\

Server begins with an arbitrary estimate $w^0 \in \W$ of the parameter $w^*$ and iteratively updates it using the following steps. We let $w^t$ denote the parameter estimate at time $t \in \Z_{\geq 0}$.

\begin{enumerate}
    \item[S1:] At each time $t \in \Z_{\geq 0}$, the server requests from each agent the gradient of its cost at the current estimate $w^{t}$, and sorts the received gradients by their norms. Let,
\[\norm{g^t_{i_1}} \leq \ldots \leq \norm{ g^t_{i_{n-f}}} \leq \ldots \leq \norm{g^t_{i_{n}}}\]
where, $i_k \in [n], \, \forall k \in [n]$ and $g_i^t$ denotes the gradient reported by agent $i$ at time $t$. Note that if $i \in \B$ then $g^t_i = \star$ (arbitrary), and if $i \in \H$ and the system is synchronous then $g^t_i = \nabla C_i(w^t)$ (asynchronous case is discussed in Assumption~\textbf{(A6)} of Section~\ref{sub:conv_2}). Let, 
\begin{align*}
    \F_t = \{i_1,\ldots, i_{n-f}\} 
\end{align*}
be the set of agents with $n-f$ smallest gradient norms at time $t$. 
    \item[S2:] The server caps the norms of the gradients reported by agents $\varrho \in [n] \setminus \F_t$ by $\norm{g^t_{i_{n-f}}}$ as 
    \begin{align}
        \overline{g^t_\varrho}  = \left\{ \begin{array}{ccc} \frac{\norm{g^t_{i_{n-f}}}}{\norm{g^t_\varrho}} \, g^t_\varrho &, &  \norm{g^t_\varrho} > 0 \\ \\ 0 & , & \text{o.w.} \end{array}\right. \label{eqn:g_cap}
    \end{align}
    and updates $w^t$ as,
\begin{align}
    w^{t+1} = \left[ w^t - \eta_{t} \cdot \left(\sum_{\sigma \in \F_t} g^t_{\sigma} + \sum_{\varrho \in [n]\setminus \F_t}  \overline{g^t_{\varrho}}\right) \right]_{\W}, \, \forall t \in \Z_{\geq 0} \label{eqn:algo_2}
\end{align}
where, $\{\eta_{t}\}$ is a sequence of bounded positive real values and $[\,\cdot \,]_\W$ denotes projection onto $\W$ w.r.t. Euclidean norm, i.e. $[w]_{\W} = \arg \min_{v \in \W} \norm{w - v}, \, \forall w \in \R^d$.
\end{enumerate}

\subsection{Modification (\emph{Informal}): Normalizing Gradients}
Instead of capping just the $f$ largest gradients, the server could scale the norms of all non-zero gradients to $\norm{g^t_{i_{n-f}}}$. In which case, the non-zero honest gradients in $\{g^t_\sigma\}_{\sigma \in \F_t}$ get amplified, whereas the maximum possible norm of Byzantine faulty agents' gradients still remains bounded by $\norm{g^t_{i_{n-f}}}$. Therefore, intuitively, correctness of Algorithm-II implies correctness of this modified version of Algorithm-II, but the other way around need not be true. However, it might be possible to improve the sufficiency bound on $f/n$ by this modification of Algorithm-II. Note that modification of Algorithm-II in this manner is equivalent to normalizing all the agents' gradients (that are non-zero), and then adding these normalized gradients to compute the update direction at each iteration. Thus, this modification replaces sorting of agents' gradients in Step S1 with normalization of agents' gradients.
%%%%%%%%%%%%%%%%%%%%%%%% CONVERGENCE %%%%%%%%%%%%%%%%%%%%%%%%%%%%%%%%
\section{Convergence Analysis: Algorithm-II}
\label{sec:conv_algo_2}
In this section, we present the convergence of Algorithm-II for the synchronous case. The convergence result is however expected to hold even under partial asynchronism.

\begin{theorem}
\label{thm:mr_algo_2}
Under assumptions \textbf{(A1)}-\textbf{(A5)}, if $\sum_{t = 0}^{\infty} \eta_t = \infty$, $\sum_{t = 0}^{\infty} \eta^2_t < \infty$, and
\begin{align}
    \frac{f}{n} < \frac{1}{2 + \mu/\gamma - \gamma/\mu} \label{eqn:cond_3}
\end{align}
then the sequence of parameter estimates $\{w^t\}$, generated by update law~\eqref{eqn:algo_2}, converges to $w^*$.
\end{theorem}
\begin{proof}
To be included in a revision of this manuscript. 
\end{proof}

Evidently, the bound on $f/n$ given in~\eqref{eqn:cond_3} is better than the bound in~\eqref{eqn:cond_2}, which was obtained for \emph{norm filtering} given in Section~\ref{sec:algo}. In fact, in an extreme case where $w^*$ is the unique minimizer of every honest agents' cost, i.e. $\mu = \gamma$, then right-hand side of~\eqref{eqn:cond_3} is equal to $1/2$. Thus, in this extreme case, Algorithm-II solves the regression problem if Byzantine faulty agents are less than the majority, which is in fact the necessary condition for solving the problem.

%%%%%%%%%%%%%%%%%%%%%%%%%%%%%%NUMERICAL EXAMPLE %%%%%%%%%%%%%%%%%%%%%%%%%%%%%%%%
\section{Numerical Example}
\label{sec:num}
In this section, we present a small numerical example to demonstrate the convergence of \emph{norm filtering} based gradient descent algorithm, as given by Theorem~\ref{thm:mr_2} for the synchronous case, i.e. under assumption~\textbf{(A4)}. \\

In this example, we choose $n = 6$, $d = 2$ and $f = 1$. Note that assumption~\textbf{(A3)} holds readily as $f<n/2$. Each agent $i \in [n]$ is associated with $n_i = 1$ data point $X_i$ and a corresponding response $Y_i$, such that 
\[Y_i = X_i w^*, \, w^* = \left[\begin{array}{c}1 \\ 1\end{array}\right], \, \forall i \in [n]\] 

\noindent The collective data points $X_{[n]}$ and responses $Y_{[n]}$ are:
\begin{align*}
    X_{[n]} = \left[\begin{array}{c} X_1 \\ X_2 \\X_3 \\ X_4\\ X_5 \\ X_6 \end{array}\right] = 
    \left[\begin{array}{cc} 1 & 0 \\ 0.8 & 0.5 \\ 0.5 & 0.8 \\ 0 & 1 \\ -0.5 & 0.8 \\ -0.8 & 0.5 \end{array}\right], \, Y_{[n]} = \left[\begin{array}{c} Y_1 \\ Y_2 \\Y_3 \\ Y_4\\ Y_5 \\ Y_6 \end{array}\right] = \left[\begin{array}{c} 1 \\ 1.3 \\ 1.3 \\ 1 \\ 0.3 \\ -0.3 \end{array}\right]
\end{align*}
For the above data points, we get the following:
\begin{enumerate}
    \item Rank of $X_{S} = [X_i]_{i \in S}$ is equal to $d = 2$ for every $S \subset [n]$ of cardinality $n-2f = 4$. This implies that assumption~\textbf{(A1)} holds with $\W = [-100, 100]^2$, and $\lambda$ is some positive real value whose exact value is not required (refer Section~\ref{sub:lin_pf} for the procedure).
    \item Assumption~\textbf{(A2)} holds and $\mu \leq 1$ (refer Section~\ref{sub:lin_pf} for the procedure).
    \item Assumption~\textbf{(A5)} holds and $\gamma \geq 0.258$ (refer Section~\ref{sub:sync} for the procedure).
\end{enumerate}
% As $f = 1$ and the identity of the Byzantine faulty agent is unknown, therefore to check assumption~\textbf{(A5)} we check the rank of every set of $n-2f = 5$ agents' data points. For the given data points, the rank of $X_{S} = [X_i]_{i \in S}$ is equal to $d = 2$ for every $S \subset [n]$ of cardinality $n-2f = 4$. Thus, assumption~\textbf{(A5)} holds and $\gamma \geq 0.258$ (refer the discussion after following assumption~\textbf{(A5)} in Section~\ref{sub:sync}). As assumption~\textbf{(A5)} holds and $w^*$ is common for all the agents, thus assumption~\textbf{(A1)} holds with $\W$ as chosen. 
Therefore,
\[\frac{1}{2 + (\mu/\gamma)} \geq 0.17\]
As $f/n = 1/6 \leq 0.167$, thus condition~\eqref{eqn:cond_2} in Theorem~\ref{thm:mr_2} is satisfied for this example.\\

We also note that Assumption 1 in Su and Shahrampour~\cite{su2018finite}, closest related work, does not hold for the given set of data points. Specifically, if $\B = \{6\}$ and $\H = \{1,2,3,4,5\}$ then
\[\frac{1}{\mnorm{\H}-\mnorm{\B}}\sum_{i \in \H}\norm{(I_2 - X_i^T X_i)e_1}_1 = 1.015 \not< 1 \text{ and } \frac{1}{\mnorm{\H}-\mnorm{\B}}\sum_{i \in \H}\norm{(I_2 - X_i^T X_i)e_2}_1 \leq 0.92 \]
where, $I_2$ is the $2 \times 2$ identity matrix, $e_1 = [1 \quad 0]^T$, $e_2 = [0 \quad 1]^T$, and $\norm{v}_1$ is the $1$-norm of any vector $v \in \R^d$, i.e
\[\norm{v}_1 = \sum_{k = 1}^d \mnorm{v[k]}\]
Thus, the proposed \emph{coordinate-wise trimmed mean} filtering technique in~\cite{su2018finite} is not guaranteed to be effective for this particular case.\\

\noindent \textbf{Omniscient Byzantine faulty agents:} To simulate our proposed algorithm, described in Section~\ref{sec:algo}, we randomly choose an agent to be Byzantine faulty. The chosen Byzantine faulty agent is assumed to have complete knowledge of honest agents' gradients, and even knows the value of $w^*$. At each time $t$, the faulty agent reports gradient that is directed opposite to $w^t - w^*$ ($w^t$ being the parameter estimate at $t$), to maximize the damage, and has norm equal to the $2$nd largest norm of honest agents' gradients to pass through the filter (as in this particular example $f = 1$ and so the filtering in step S1 eliminates the gradient with largest norm). \\

Expectedly (cf. Theorem~\ref{thm:mr_2}), the proposed algorithm converges to $w^*$ for this example with $w^0 = [0 \quad 0]^T$ and step-size $\eta_t = 10/(t+1), \, \forall t \in \Z_{\geq 0}$, regardless of the identity of Byzantine faulty agent. Note that $\sum_{t = 0}^\infty\eta_t = \infty$ and $\sum_{t = 0}^\infty\eta^2_t < \infty$ (refer.~\cite{rudin1964principles}). \\

Convergence plot of the proposed (with \emph{norm filtering}) gradient descent algorithm (plotted in `blue') for $\B = \{2\}$ (chosen randomly for the purpose of simulation) is shown in Figure~\ref{fig:small-ex}. In the plot, the estimation error is equal to $\norm{w^t - w^*}$ for each iteration (or time) $t \in [0, \, 50)$. The initial estimate $w^0 = [0 \quad 0]^T$, Byzantine faulty agent is omniscient and chooses its gradients as described above. \\

\begin{figure}[htb!]
\centering
\includegraphics[width=0.75\textwidth]{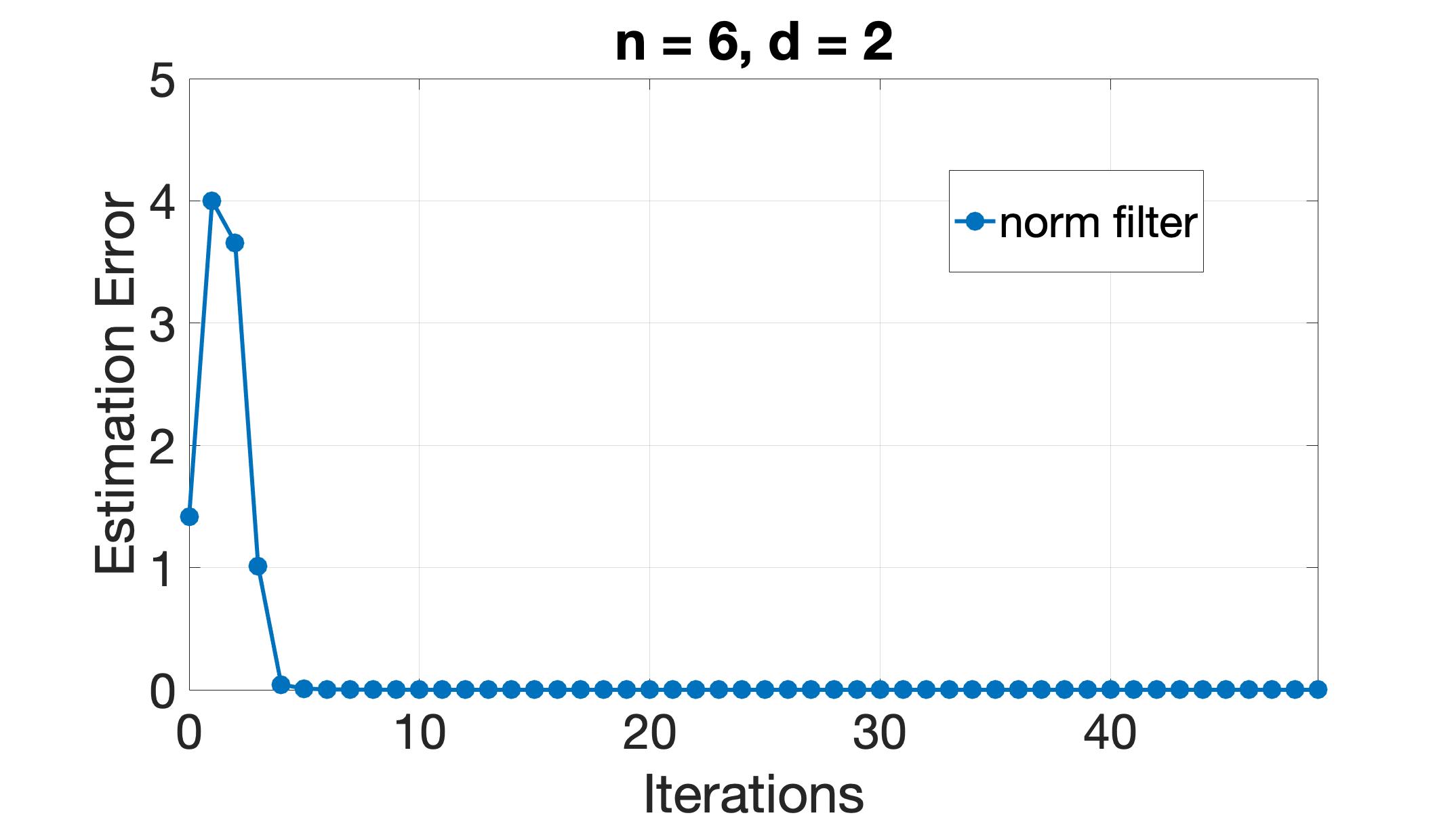} 
\caption{\small{\it Here, the estimation error is $\norm{w^t - w^*}$ at each iteration $t$ and $\B = \{2\}$. The Byzantine faulty agent is assumed omniscient and chooses its gradients as described above. The plot corresponds to the estimation errors for the \emph{norm filtering} based gradient descent algorithm, given in Section~\ref{sec:algo}. The initial estimate $w^0 = [0 \quad 0]^T$.}}
\label{fig:small-ex}
\end{figure}

\noindent \textbf{Ill-informed Byzantine faulty agents:} It may happen that Byzantine faulty agents are not omniscient, as mentioned above. They could just have access to information held by them. To simulate such faulty behavior, in this example, the Byzantine faulty agent simply reports randomly chosen gradient vectors to the server in step S1. The proposed norm filter converges to $w^*$, as expected (shown in Figure~\ref{fig:small-ex-2}). Whereas, the original gradient descent algorithm does not converge as expected, and often diverges away from $w^*$ as shown in Figure~\ref{fig:small-ex-2}.

\begin{figure}[htb!]
\centering
\includegraphics[width=0.75\textwidth]{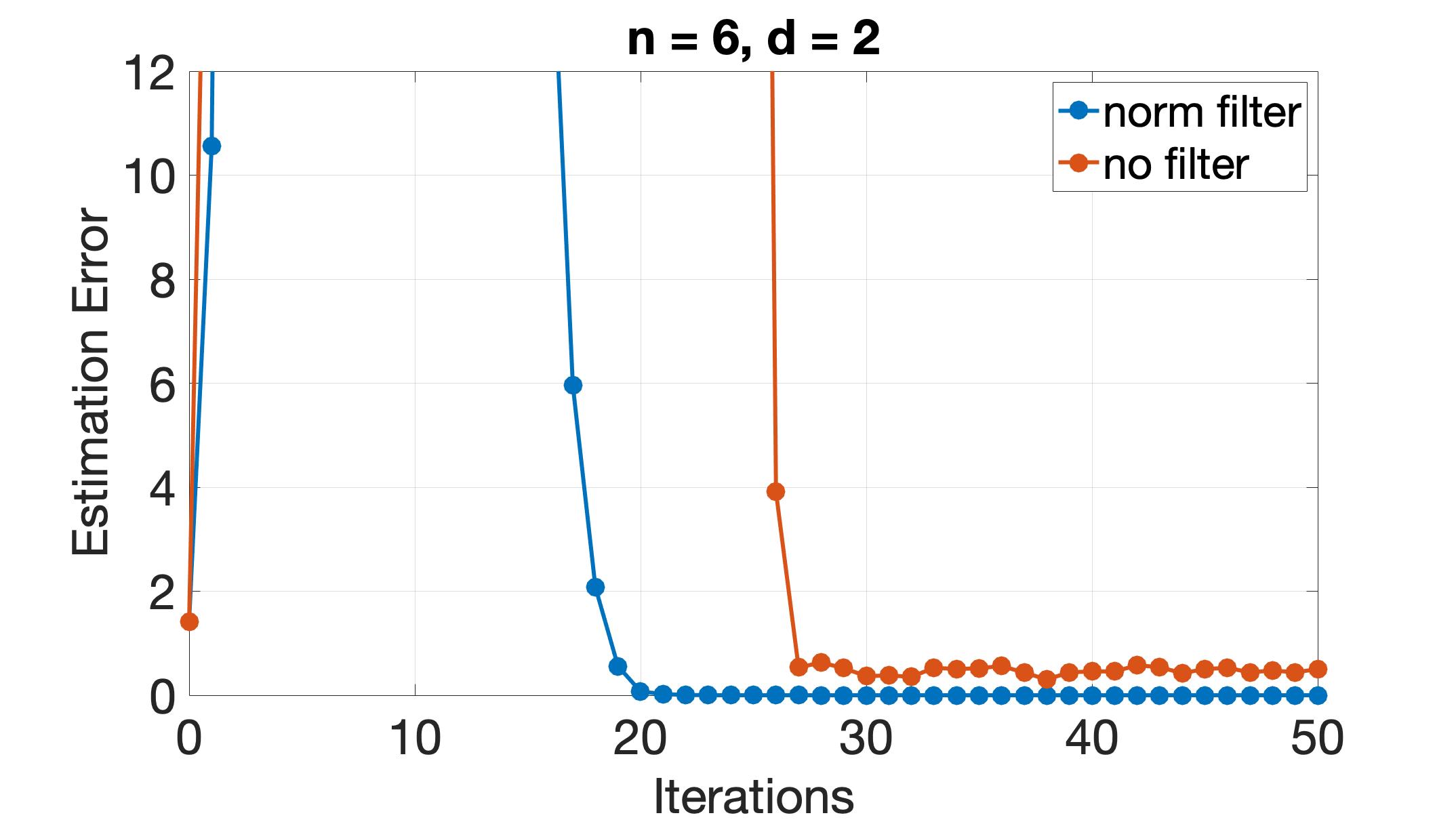} 
\caption{\small{\it Here, the estimation error is $\norm{w^t - w^*}$ at each iteration $t$ and $\B = \{2\}$. The Byzantine faulty agent is assumed ill-informed and chooses its gradients randomly, as described above. The plots in `blue' and `red' correspond to the estimation errors of the \emph{norm filtering} based gradient descent algorithm (ref. Section~\ref{sec:algo}) and the original gradient descent algorithm (without any filtering), respectively. For both the algorithms, the initial estimate $w^0 = [0 \quad 0]^T$.}}
\label{fig:small-ex-2}
\end{figure}

\section{Conclusion}
\label{sec:conclude}
This paper proposes two simple norm based filtering techniques, \emph{norm filtering} and \emph{norm-cap filtering}, for ``robustifying" the original distributed gradient descent algorithm for solving distributed linear regression problem in presence of Byzantine faulty agents in the multi-agent system, when the maximum possible number of Byzantine faulty agents is less than a specified bound. The proposed ``robustification" techniques also solve a more general multi-agent optimization problem with Byzantine faults. We note that the obtained bound on the number of faulty agents, which if satisfied guarantees correctness of the proposed algorithm, relates to the conditioning of the resultant matrix constructed by stacking the data points of the honest agents. \\

% The proposed algorithm does not require the set of Byzantine faulty agents to be fixed,  i.e. sets $\H$ and $\B$ can change over time, as long as there are at most $f$ Byzantine faulty agents and honest agents behave as per the assumptions made in Section~\ref{sec:pf} at all times~\cite{xie2018generalized}. To the best of authors' knowledge, the proposed algorithm is the first ever (including the literature on centralized setting as well) to solve distributed regression problem with \emph{generalized} Byzantine faulty agents\footnote{Authors in \cite{xie2018generalized} proposes a \emph{generalized} Byzantine fault tolerant parallelization of stochastic gradient descent, where all the agents have access to a common set of data points and responses.}.
%Our proposed algorithm can not tolerate more than or equal to $n/3$ Byzantine faulty agents. Understandably, this corresponds to the well-known necessary condition in Byzantine fault tolerant consensus problem~\cite{lynch1996distributed}.

% For now, we do not consider noise in the responses of honest agents. However, the resilience of our filtering technique to bounded delays in the system (a.k.a. \emph{partial asynchronism}) suggests robustness of our algorithm to additive noise of bounded variance and zero mean. Moreover, we believe that tolerance bound on $f/n$ won't change as far as almost surely convergence is concerned.\\

\textbf{Stopping Failures:} Even though the proposed algorithm can handle any kind of faults, including stopping failure (when a certain agent crashes and stops responding), it is not yet optimal for handling such inadvertent crashes. However, the server can simply define an upper limit on the outdatedness (time passed since the last update) of an agent's gradient and deem a particular agent as `crashed' if the outdatedness of the agent's gradient exceeds the limit.

\section*{Acknowledgements}
Research reported in this paper was sponsored in part by the Army Research Laboratory under Cooperative Agreement W911NF- 17-2-0196, and by National Science Foundation award 1610543. The views and conclusions contained in this document are those of the authors and should not be interpreted as representing the official policies, either expressed or implied, of the the Army Research Laboratory, National Science Foundation or the U.S. Government.

\bibliographystyle{IEEEtran}
\bibliography{ref_resilient_dist_opt}

%\appendix
%\addappheadtotoc
\begin{appendices}

%%%%%%%%%%%%%%%%%%%%%%%%%% NOISE %%%%%%%%%%%%%%%%%%%%%%%%%%%%%%
\section{Appendix: Noisy Gradients}
\label{app:noise}
In practice, honest agents might not report their costs' gradients accurately due to reasons such as system noise or quantization errors. Specifically, in case of synchronous execution we assume the following.
\begin{itemize}
    \item[\textbf{(A7)}] \textbf{Noisy Gradients}: For each honest agent $i \in \H$, assume that
    \[g^t_i = \nabla C_i (w^t) + D_i(w^t), \, \forall t \in \Z_{\geq 0}\]
    where, $\norm{D_i(w)} \leq D < \infty, \, \forall w \in \R_{\geq 0}$.
\end{itemize}

\subsection{Noisy Responses in Linear Regression}
The above approximate gradient framework models the case of noisy responses in distributed linear regression, where
\begin{align}
    Y_i = X_i w^* + \xi_i, \, \norm{\xi_i} \leq \xi < \infty, \quad \forall i \in \H \label{eqn:noise_resp}
\end{align}
The actual error cost of an agent $i \in \H$ at an estimated parameter value $w \in \R^d$ is
\begin{align}
    C_i(w) = (1/2)\norm{X_i w - X_i w^*}^2 \label{eqn:actual_cost}
\end{align}
However, agent $i \in \H$ can only observe $Y_i$, and not $X_i w^*$. Therefore, the error cost observed by agent $i \in \H$ at an estimated parameter value $w \in \R^d$ is
\[\widehat{C}_i(w) = (1/2)\norm{X_i w - X_i w^*}^2\]
Thus, the reported gradient $g^t_i$ of an agent $i \in \H$ at any time $t \in \Z_{\geq 0}$, in Step S1 of the Algorithm given in Section~\ref{sec:algo}, is given as follows (for the synchronous case).
\[g^t_i = \nabla \widehat{C}_i(w^t) = X_i^T (X_i w^t - Y_i)\]
Substituting~\eqref{eqn:noise_resp} above gives
\[g^t_i = X_i^T X_i (w^t - w^*) - X_i^T \xi_i\]
As $\nabla C_i(w) = X_i^T X_i (w^t - w^*), \, \forall w \in \R^d$ (cf.~\eqref{eqn:actual_cost}), thus for the synchronous case,
\[g^t_i = \nabla C_i(w^t) - X_i^T \xi_i, \quad \forall i \in \H, \, \forall t \in \Z_{\geq 0}\]
Note that the above gradient is a special case of the noisy gradient model in Assumption~\textbf{(A7)}, where $D_i(w^t) = - X_i^T \xi_i, \, \forall i \in \H, \, \forall t \in \Z_{\geq 0}$. As $\norm{\xi_i} \leq \xi, \, \forall i \in \H$, thus 
\[\norm{D_i(w^t)} = \sqrt{\xi^T_i \left(X_i X^T_i\right) \xi_i} \leq \sqrt{u_i} \norm{\xi_i} \leq \sqrt{u_i} \, \xi, \, \forall t \in \Z_{\geq 0}, \, \forall i \in \H\]
where, $u_i$ is the largest eigenvalue of positive semi-definite matrix $X_i X^T_i$. Let $u = \max_{i \in \H}\{\sqrt{u_i}\}$, then
\[\norm{D_i(w^t)} \leq u \,  \xi < \infty, \, \forall t \in \Z_{\geq 0}, \, \forall i \in \H\]

\subsection{Convergence Analysis: Algorithm-I With System Noise}
Intuitively, it is impossible in general for any algorithm to compute $w^*$ accurately when none of the agents report gradients of their costs accurately. However, if the algorithm is robust enough then it can compute a point in the neighborhood of $w^*$, whose size usually depends on the magnitude of inaccuracies (or noise) in the agents' gradients. For the proposed algorithm with update law~\eqref{eqn:algo_1} in Section~\ref{sec:algo}, we can guarantee convergence to a neighborhood of $w^*$ whose size, expectedly, depends on $D$ and the also on the fraction of maximum possible Byzantine faulty agents $f/n$.
\begin{theorem}
\label{thm:mr_noise}
Consider the update law~\eqref{eqn:algo_1} given in Section~\ref{sec:algo} under assumptions \textbf{(A1)}-\textbf{(A3)}, \textbf{(A5)} and \textbf{(A7)}. If $\sum_{t = 0}^{\infty} \eta_t = \infty$, $\sum_{t = 0}^{\infty} \eta^2_t < \infty$, and condition~\eqref{eqn:cond_2} holds then for 
\[D^* = \frac{1}{\gamma}\left( \frac{1 - 2(f/n)}{1 - (f/n)(2 + \mu/\gamma)} \right) D\]
there exists a finite $\tau \in \Z_{\geq 0}$ such that 
\[\norm{w^t - w^*} \leq D^*, \, \forall t \geq \tau\]
\end{theorem}
\begin{proof}
Refer Appendix~\ref{app:mr_noise}. 
\end{proof}
Theorem~\ref{thm:mr_noise} states that the final inaccuracy of the solution obtained by the server using the algorithm given in Section~\ref{sec:algo} can be at most $D^*$ w.r.t $2$-norm. In case $f = 0$,
\[D^* = \left(\frac{1}{\gamma} \right) D\]
For now, we have only considered the synchronous case. However, using similar arguments as in assumption~\textbf{(A6)} and Theorem~\ref{thm:mr_delay}, the above convergence result is expected to hold even when there is partial asynchronicity in the system.

\section{Appendix: Proofs}
\label{app:proofs}

\subsection{Proof of Claim~\ref{clm:red}}
\label{sub:red}

% Let $i_1$ and $i_2$ be any two \emph{different} agents in $\H$. As $f \geq 1$, therefore we can choose a set $\hat{\H}$ of $n-2f-1$ agents from the remaining set of honest agents $\H \setminus \{i_1,\, i_2\}$.\\
% Let $\H^1 = \hat{\H} \cup \{i_1\}$ and $\H^2 = \hat{\H} \cup \{i_2\}$. As $\mnorm{\H^1} = \mnorm{\H^2}$, thus from assumption \textbf{(A3)},
% \begin{align}
%     \sum_{i \in \hat\H}\nabla C_{i}(w^*) + \nabla C_{i_1}(w^*)= \sum_{i \in \hat\H}\nabla C_{i}(w^*) + \nabla C_{i_2}(w^*) = 0 \label{eqn:a3}
% \end{align}
% Thus,
% \[\nabla C_{i_1}(w^*) = \nabla C_{i_2}(w^*)\]
% As the above holds for any two $i_1, \, i_2 \in \H$, therefore the value of $\nabla C_i(w^*)$ is same for every $i \in \H$. Using this deduction in \eqref{eqn:a3} implies, $\nabla C_i(w^*) = 0, \, \forall i \in \H$.\\

% From assumption~\textbf{(A1)}, $w^*$ minimizes $C_{\H}$ (cf. \eqref{eqn:obj}) and therefore, $\nabla C_{\H}(w^*) = 0$ as $C_{\H}$ is convex (functions $C_i, \, \forall i \in \H$ are convex as per assumption~\textbf{(A2)}). Combining this with the fact that $C_{\H}$ is a positive function (as $C_i, \, \forall i \in \H$ are all positive functions) implies that $C_{\H}(w^*) = 0$. \\

% As $C_{\H}(w) = (1/\mnorm{\H})\sum_{i \in \H}C_i(w)$, where $C_i, \, \forall i \in \H$ are all positive functions, therefore $C_i(w^*) = 0, \, \forall i \in \H$. \\
% This implies that $w^*$ also minimizes $C_i, \, \forall i \in \H$. 

As $C_i$ is convex for all $i \in \H$, thus assumption~\textbf{(A2)} implies
\[\nabla C_i(w^*) = 0, \, \forall i \in \H\]
Lipschitz continuity (assumption~\textbf{(A2)}) of $\nabla C_i, \, \forall i \in \H$ further implies
\[\norm{\nabla C_i (w)} \leq \mu \norm{w - w^*}, \, \forall w \in \R^d, \, \forall i \in \H \]
Combining this inequality with Cauchy-Schwartz inequality implies,
\begin{align}
    \iprod{w - w^*}{\nabla C_i (w)} \leq \mu \norm{w - w^*}^2, \, \forall w \in \R^d, \, \forall i \in \H \label{eqn:ci_bnd}
\end{align}
From assumption \textbf{(A1)},
\begin{align}
    \iprod{w-w^{*}}{\nabla C_{\H}(w)} = \frac{1}{\mnorm{\H}}\sum_{i \in \H} \iprod{w - w^*}{\nabla C_{i}(w)} \geq \lambda \norm{w - w^*}^2 , \, \forall w \in \R^d \label{eqn:ch_bound}
\end{align}
as $\nabla C_{\H}(w^*) = 0$. Therefore,~\eqref{eqn:ci_bnd} and~\eqref{eqn:ch_bound} imply that
\begin{align*}
    \mu \norm{w - w^*}^2  \geq \lambda \norm{w - w^*}^2, \, \forall w \in \R^d
\end{align*}
Hence, $\mu \geq \lambda$ if assumptions \textbf{(A1)}-\textbf{(A2)} hold.\\

%Implication \eqref{eqn:restrict} is shown by combining inequality \eqref{eqn:ci_bnd} with assumption \textbf{(A1)} as follows. \\
Now, the above implies that if $f/n < 1 / (1 + (\mu/\lambda))$ then $n > 2f$. So, let $\H'$ be a non-empty subset of $\H$ such that $\mnorm{\H'} = n-2f$, and let $\hat{\H}$ be a subset of $\H$, such that $\H' \subset \hat\H$ and $\mnorm{\hat\H} = n-f$. Then,
\begin{align*}
    \nabla C_{\hat{\H}}(w) = \frac{|\H'|}{|\hat{\H}|} \nabla C_{\H'}(w) + \frac{1}{|\hat{\H}|} \sum_{j \in \hat{\H} \setminus \H'} \nabla C_{j}(w)
\end{align*}
From above,
\begin{align*}
    \iprod{w - w^*}{\nabla C_{\hat{\H}}(w)} = \frac{|\H'|}{|\hat{\H}|} \iprod{w-w^*}{\nabla C_{\H'}(w)} + \frac{1}{|\hat{\H}|} \sum_{j \in \hat{\H} \setminus \H'} \iprod{w-w^*}{\nabla C_{j}(w)}
\end{align*}
Using~\eqref{eqn:ci_bnd} above implies,
\begin{align*}
    \iprod{w - w^*}{\nabla C_{\hat{\H}}(w)} \leq \frac{|\H'|}{|\hat{\H}|} \iprod{w-w^*}{\nabla C_{\H'}(w)} +  \frac{1}{|\hat{\H}|}\sum_{j \in \hat{\H}\setminus \H'}\mu \norm{w - w^*}^2
\end{align*}
Assumption~\textbf{(A1)} implies,
\begin{align*}
    \lambda \norm{w - w^*}^2 \leq \iprod{w - w^*}{\nabla C_{\hat{\H}}(w)} \leq \frac{|\H'|}{|\hat{\H}|} \iprod{w-w^*}{\nabla C_{\H'}(w)} +  \frac{1}{|\hat{\H}|}\sum_{j \in \hat{\H}\setminus \H'}\mu \norm{w - w^*}^2
\end{align*}
Therefore,
\begin{align*}
    \frac{n-2f}{|\hat{\H}|}\iprod{w - w^*}{\nabla C_{\H'} (w)} + \frac{1}{|\hat{\H}|}\sum_{j \in \hat{\H}\setminus \H'}\mu \norm{w - w^*}^2 \geq \lambda \norm{w - w^*}^2 , \, \forall w \in \W
\end{align*}
Or,
\begin{align*}
    & \frac{n-2f}{|\hat{\H}|}\iprod{w - w^*}{\nabla C_{\H'} (w)} + \frac{(|\hat{\H}|-n+2f)\mu}{|\hat{\H}|} \norm{w - w^*}^2 \geq \lambda \norm{w - w^*}^2,\\
    \implies & (n - 2f) \cdot \iprod{w - w^*}{\nabla C_{\H'} (w)} \geq (|\hat{\H}| (\lambda - \mu) + (n-2f) \mu ) \cdot \norm{w - w^*}^2, \, \forall w \in \R^d
\end{align*}
Substituting $\mnorm{\hat{\H}} = n-f$ above implies 
\begin{align*}
 	(n - 2f) \cdot \iprod{w - w^*}{\nabla C_{\H'} (w)} \geq (n \lambda - f (\lambda + \mu)) \cdot \norm{w - w^*}^2, \, \forall w \in \R^d
\end{align*}
As $n-2f > 0$ (argued above), therefore, if 
\[\frac{f}{n} < \frac{1}{1 + (\mu/\lambda)}\]
then
\[\iprod{w - w^*}{\nabla C_{\H'} (w)} \geq \zeta \norm{w - w^*}^2, \, \forall w \in \R^d \]
where, 
\[\zeta  = \frac{n \lambda - f (\lambda + \mu)}{n-2f} > 0\]
Thus, $\nabla C_{\H'}(w) = 0$ only if $w = w^*$. From assumption~\textbf{(A2)}, we have
\[\nabla C_{\H'} (w^*) = \frac{1}{\mnorm{\H'}}\sum_{i \in \H' \subset \H} \nabla C_i(w^*) = \frac{1}{\mnorm{\H'}}\sum_{i \in \H' \subset \H} 0 = 0\]
Hence, the above implies that $\nabla C_{\H'}(w) = 0$ iff $w = w^*$ when assumptions~\textbf{(A1)}-\textbf{(A2)} hold and $f/n < 1/(1 + (\mu/\lambda))$.

\subsection{Proof of Claim~\ref{clm:mu_gamma}}
\label{sub:mu_gamma}
As $C_i$ is convex for all $i \in \H$, thus assumption~\textbf{(A2)} implies
\[\nabla C_i(w^*) = 0, \, \forall i \in \H\]
Lipschitz continuity (assumption~\textbf{(A2)}) of $\nabla C_i, \, \forall i \in \H$ further implies
\[\norm{\nabla C_i (w)} \leq \mu \norm{w - w^*}, \, \forall w \in \R^d, \, \forall i \in \H \]
Combining this inequality with Cauchy-Schwartz inequality implies,
\begin{align}
    \iprod{w - w^*}{\nabla C_i (w)} \leq \mu \norm{w - w^*}^2, \, \forall w \in \R^d, \, \forall i \in \H \label{eqn:ci_bnd_2}
\end{align}
For any subset $\H' \subset \H$ of cardinality $n-2f$ (note that $\H'$ is non-empty as $f < n/2$ due to assumption~\textbf{(A3)}),
\[\nabla C_{\H'}(w) = \frac{1}{\mnorm{\H'}}\sum_{i \in \H'} \nabla C_i(w)\]
Thus, $\nabla C_{\H'}(w^*) = 0$. Therefore, assumption \textbf{(A5)} implies that
\begin{align}
    \iprod{w-w^{*}}{\nabla C_{\H'}(w)} = \frac{1}{\mnorm{\H'}}\sum_{i \in \H'} \iprod{w - w^*}{\nabla C_{i}(w)} \geq \gamma \norm{w - w^*}^2, \, \forall w \in \R^d \label{eqn:ch'_bound}
\end{align}
Therefore,~\eqref{eqn:ci_bnd_2} and~\eqref{eqn:ch'_bound} imply that
\begin{align*}
    \mu \norm{w - w^*}^2  \geq \gamma \norm{w - w^*}^2, \, \forall w \in \R^d
\end{align*}
Hence, $\mu \geq \gamma$ if assumptions \textbf{(A1)}-\textbf{(A2)} hold.

\subsection{Proof of Theorem~\ref{thm:mr_1}}
\label{sub:mr_1}

Define $h_t = \norm{w^t - w^{*}}^2$. From \eqref{eqn:algo_1}, we get
\[ h_{t+1} = \norm{\left[ w^t - \eta_{t} \cdot \sum_{\sigma \in \F_t} g^t_{\sigma} \right]_{\W} - w^{*}}^2 \]
Due to the non-expansion property of projection onto a closed convex set~\cite{boyd2004convex}, $\norm{w - w^{*}} \geq \norm{[w]_{\W} - w^{*}}, \, \forall w \in \R^{d}$, therefore 
\begin{align}
	h_{t+1} \leq \norm{w^t - \eta_{t} \cdot \sum_{\sigma \in \F_t} g^t_{\sigma} - w^{*}}^2 = h_t - 2\eta_t \iprod{w^t - w^{*}}{\sum_{\sigma \in \F_t} g^t_{\sigma}} + \eta^2_t \norm{\sum_{\sigma \in \F_t} g^t_{\sigma}}^2 \label{eqn:ht_1}
\end{align}
As there are at most $f$ Byzantine agents, thus for each time $t \in \Z_{\geq 0}$ there exists $j_t \in \H$ such that
\[\norm{g^t_{\sigma}} \leq \norm{g^t_{i_{n-f}}} \leq \norm{g^t_{j_t}}, \, \forall \sigma \in \F_t\] 
%using the definition of $\{\overline{g^t_\sigma}\}$ in~\eqref{eqn:filt}.
From assumption \textbf{(A4)}, $g^t_{j_t} = \nabla C_{j_t}(w^{t})$, thus
\begin{align}
	\norm{g^t_{\sigma}} \leq \norm{\nabla C_{j_t}(w^{t})}, \, \forall \sigma \in \F_t \label{eqn:c_jt}
\end{align}
As $w^*$ is assumed to be a minimizer of all the honest agents' cost (cf. assumption~\textbf{(A2)}), 
\[\nabla C_{j_t}(w^*) = 0\]
Therefore, assumption~\textbf{(A2)} implies,
\begin{align}
	\norm{\nabla C_{j_t}(w)} \leq \mu \norm{w - w^*}, \, \forall w \in \W \label{eqn:cj_bound}
\end{align}
Let $\Gamma = \max_{w \in \W}\norm{w - w^*}$, where $\W$ is a compact set in $\R^d$. It should be noted that $\Gamma < \infty$ as compact sets in real spaces are bounded. Thus,
\begin{align}
	\norm{\nabla C_{j_t}(w)} \leq \mu \max_{w \in \W}\norm{w - w^*} = \mu \Gamma, \, \forall w \in \W \label{eqn:c_jt_2}
\end{align}
From triangle inequality, \eqref{eqn:c_jt} and \eqref{eqn:c_jt_2} we obtain,
\[\norm{\sum_{\sigma \in \F_t} g^t_{\sigma}} \leq \sum_{\sigma \in \F_t}\norm{g^t_{\sigma}} \leq (n-f) \mu \Gamma\]
Substituting this in \eqref{eqn:ht_1} implies,
\begin{align}
	h_{t+1} \leq h_t - 2\eta_t \iprod{w^t - w^{*}}{\sum_{\sigma \in \F_t} g^t_{\sigma}} + (n-f)^2 \mu^2 \Gamma^2 \eta^2_t \label{eqn:ht_before_split}
\end{align}
% Let $\F_t = \{w^0,\ldots,w^t,\sigma_0,\ldots, \sigma_{t-1}\}$. Then,
% \begin{align*}
% \E(h_{t+1}|\F_t) \leq h_t - 2\eta_t \E(\iprod{w^t - w^{*}}{\overline{g^t_{\sigma_t}}}|\F_t) + \eta^2_t \Gamma^2
% \end{align*}
As $\mnorm{\F_t} = n-f$ and it is assumed that $f<n/2$ (assumption~\textbf{(A3)}), therefore there exists a subset $\H^t_1 \subset \F_t$ of cardinality $n-2f$ such that $\H^t_1 \subset \H$. Thus, using assumption~\textbf{(A4)} we get
\[g^t_{i} = \nabla C_i(w^t), \, \forall i \in \H^t_1, \, \forall t \in \Z_{\geq 0}.\]
Substituting this in~\eqref{eqn:ht_before_split} gives,
\begin{align}
\begin{split}
h_{t+1} & \leq h_t - 2\eta_t \sum_{i \in \H^t_1}\iprod{w^t - w^{*}}{\nabla C_i(w^t)} - 2\eta_t \sum_{k \in \F_t \setminus \H^t_1}\iprod{w^t-w^{*}}{g^t_k} + (n-f)^2 \mu^2 \Gamma^2 \eta^2_t \\
& = h_t - 2\eta_t \mnorm{\H^t_1} \iprod{w^t - w^{*}}{\nabla C_{\H^t_1}(w^t)} - 2\eta_t \sum_{k \in \F_t \setminus \H^t_1}\iprod{w^t-w^{*}}{g^t_k} + (n-f)^2 \mu^2 \Gamma^2 \eta^2_t \label{eqn:ht_split}
\end{split}
\end{align}
where, $C_{\H^t_1} = (1/\mnorm{\H^t_1})\sum_{i \in \H^t_1} C_i$. Using Cauchy-Schwartz inequality, we know that
\begin{align*}
    \mnorm{\iprod{w^t-w^{*}}{g^t_k}} &\leq \norm{w^t-w^{*}} \cdot \norm{g^t_k}, \, \forall k \in \F_t \setminus \H^t_1
\end{align*}
Using the fact that $\norm{g^t_\sigma} \leq \norm{\nabla C_{j_t}(w^t)}, \, \forall \sigma \in \F_t$, we get
\[\mnorm{\iprod{w^t-w^{*}}{g^t_k}} \leq \norm{w^t-w^{*}} \cdot \norm{\nabla C_{j_t}(w^t)},  \, \forall k \in \F_t \setminus \H^t_1\]
Thus,
\[\iprod{w^t-w^{*}}{g^t_k} \geq - \norm{w^t-w^{*}} \cdot \norm{\nabla C_{j_t}(w^t)},  \, \forall k \in \F_t \setminus \H^t_1\]
As $\mnorm{\H^t_1} = n-2f$ and $\mnorm{\F_t} = n-f$, thus, 
\begin{align}
    \sum_{k \in \F_t \setminus \H^t_1}\iprod{w^t-w^{*}}{g^t_k} \geq -f \cdot \norm{w^t-w^{*}} \cdot \norm{\nabla C_{j_t}(w^t)} \label{eqn:bf_bound}
\end{align}
From substituting $\mnorm{\H^t_1} = n-2f$ and~\eqref{eqn:bf_bound} in~\eqref{eqn:ht_split} we obtain, 
\begin{align*}
        h_{t+1} \leq h_t - 2\eta_t \left\{ (n-2f) \cdot \iprod{w^t - w^{*}}{\nabla C_{\H^t_1}(w^t)} - f \cdot \norm{w^t-w^{*}} \cdot \norm{\nabla C_{j_t}(w^t)} \right\} + (n-f)^2 \mu^2 \Gamma^2 \eta^2_t
\end{align*}
If we let 
\begin{align}
	\phi_t = (n-2f) \cdot \iprod{w^t - w^{*}}{\nabla C_{\H^t_1}(w^t)} - f \cdot \norm{w^t-w^{*}} \cdot \norm{\nabla C_{j_t}(w^t)} \label{eqn:phi_t_begin}
\end{align}
then the last inequality can be written as 
\begin{align}
    h_{t+1} \leq h_t - 2\eta_t \phi_t + \eta_t^2 (n-f)^2 \mu^2 \Gamma^2 \label{eqn:ineq_1}
\end{align}
Now, consider two possible cases; case (i) $w^t = w^*$, and case (ii) $w^t \neq w^*$.\\
Case (i) If $w^t = w^*$ then from assumption~\textbf{(A2)},
\[\nabla C_{j_t}(w^t) = 0\]
and
\[\nabla C_{\H^t_1}(w^t) = \frac{1}{\mnorm{\H^t_1}} \sum_{i \in \H^t_1} \nabla C_i(w^t) = 0\]
Using the above inferences in~\eqref{eqn:phi_t_begin} imply that
\begin{align}
    \phi_t = 0 \text{ if } w^t = w^* \label{eqn:phi_0}
\end{align}
Case (ii) Let $w^t \neq w^*$. From \eqref{eqn:cj_bound} we obtain,
\[ \norm{\nabla C_{j_t}(w^t)} \leq \mu \norm{w^t - w^{*}} \]
Therefore (cf.~\eqref{eqn:phi_t_begin}),
\begin{align}
    \phi_t \geq (n-2f) \cdot \iprod{w^t - w^{*}}{\nabla C_{\H^t_1}(w^t)} - \mu f \cdot \norm{w^t-w^{*}}^2 \label{eqn:phi_1}
\end{align}
Let $\H^t \subseteq \H$ of a set of $n-f$ honest agents such that $\H^t_1 \subset \H^t$. Therefore,
\[\nabla C_{\H^t}(w^t)  = \frac{\mnorm{\H^t_1}}{\mnorm{\H^t}} \nabla C_{\H^t_1}(w^t) + \frac{1}{\mnorm{\H^t}}\sum_{j \in \H^t_2 } \nabla C_{j} (w^t) \]
where, $\H^t_2 = \H^t \setminus \H^t_1$.
%\footnote{For any $\hat{\H} \subseteq \H$ with $\mnorm{\hat{\H}} = n-f$ containing $\H'$, $\nabla C_{\hat{\H}}(w) = \frac{|\H^t_1|}{|\hat{\H}|} \nabla C_{\H^t_1}(w) + \frac{1}{|\hat{\H}|} \sum_{j \in \H^t_2} \nabla C_{j}(w), \, \forall w \in \W$.}
From assumption \textbf{(A1)}, 
\[\iprod{w^t - w^*}{\nabla C_{\H^t}(w^t)} \geq \lambda \norm{w^t - w^*}^2, \, \forall t \in \Z_{\geq 0}\]
Thus (substituting $\mnorm{H^t_1} = n-2f$),
\begin{align}
    \frac{n-2f}{|\H^t|}\iprod{w^t - w^*}{\nabla C_{\H^t_1} (w^t)} + \frac{1}{|\H^t|}\sum_{j \in \H^t_2}\iprod{w^t - w^*}{\nabla C_{j} (w^t)} \geq \lambda \norm{w^t - w^*}^2 \label{eqn:h1_bnd}
\end{align}
From Cauchy-Schwartz inequality, 
\[\iprod{w^t - w^*}{\nabla C_{j} (w^t)} \leq \norm{w^t - w^*} \cdot\norm{\nabla C_{j} (w^t)}, \, \forall j \in \H^t_2\] 
Due to assumption~\textbf{(A2)},
\[\norm{\nabla C_{j} (w^t)} \leq \mu \norm{w^t - w^{*}}, \, \forall j \in \H^t_2\]
Therefore, 
\[\iprod{w^t - w^*}{\nabla C_{j} (w^t)} \leq \mu \norm{w^t - w^*}^2, \, \forall j \in \H^{t}_2\]
Using the above inequality in \eqref{eqn:h1_bnd} implies
\begin{align*}
    \frac{n-2f}{|\H^t|}\iprod{w^t - w^*}{\nabla C_{\H^t_1} (w^t)} + \frac{1}{|\H^t|}\sum_{j \in \H^t_2}\mu \norm{w^t - w^*}^2 \geq \lambda \norm{w^t - w^*}^2
\end{align*}
As $\mnorm{\H_1^t} = n - 2f$ and $\H^t_2 = \H^t \setminus \H^t_1$, thus $\mnorm{\H^t_2} = \mnorm{\H^t} - n + 2f, \, \forall t \in \Z_{\geq 0}$. Thus, from above,
\begin{align*}
    & \frac{n-2f}{|\H^t|}\iprod{w^t - w^*}{\nabla C_{\H^t_1} (w^t)} + \frac{(|\H^t|-n+2f)\mu}{|\H^t|} \norm{w^t - w^*}^2 \geq \lambda \norm{w^t - w^*}^2,\\
    & \implies (n - 2f) \cdot \iprod{w^t - w^*}{\nabla C_{\H^t_1} (w^t)} \geq (|\H^t| (\lambda - \mu) + (n-2f) \mu ) \cdot \norm{w^t - w^*}^2
\end{align*}
Using this inequality in \eqref{eqn:phi_1} implies
\begin{align*}
    \phi_t \geq (|\H^t| (\lambda - \mu) + (n-3f) \mu) \norm{w^t - w^*}^2
\end{align*}
%As $\lambda \leq \mu$ (ref. Claim \ref{clm:red}), thus $(\lambda - \mu \leq 0)$. 
Substituting $|\H^t| = n-f$ above implies,
\begin{align*}
    \phi_t \geq (\lambda n - f (\lambda + 2 \mu)) \norm{w^t - w^*}^2, \, \forall w^t \in \W 
\end{align*}
Therefore, if condition \eqref{eqn:cond_1} holds then for any positive real value $\delta$,
\begin{align}
    \text{ if } \norm{w^t - w^*}^2 > \delta \text{ then } \phi_t > (\lambda n - f (\lambda + 2 \mu))\delta > 0 \label{eqn:phi_2}
\end{align}
%Thus, if \eqref{eqn:cond_1} holds, then $\phi_t > 0$ when $w^t \neq w^*$.
% (combine case (i) and (ii)), 
% \begin{align}
%     \phi_t = \left\{ \begin{array}{ccc} 0 & , & w^t \in \W^{*} \\ > 0 & , & w^t \in \W \setminus \W^{*} \end{array}\right.  \label{eqn:phi_3}
% \end{align}
Owing to condition~\eqref{eqn:cond_1}, \eqref{eqn:ineq_1}, \eqref{eqn:phi_0} and \eqref{eqn:phi_2}, we get
\begin{align}
     h_{t+1} \leq h_t - 2\eta_t \phi_t + \eta^2_t (n-f)^2 \mu^2 \Gamma^2 \leq h_t + \eta^2_t (n-f)^2 \mu^2 \Gamma^2, \quad \forall t \in \Z_{\geq 0} \label{eqn:final}
\end{align}
Therefore,
\[(h_{t+1} - h_t)_{+} \leq \eta^2_t (n-f)^2 \mu^2 \Gamma^2, \quad \forall t \in \Z_{\geq 0}\]
where, operator $(\cdot)_{+}$ is same as defined in Lemma \ref{lem:seq_conv}. As $\sum_{t = 0}^{\infty}\eta^2_t < \infty$ and $\Gamma < \infty$, therefore
\[ \sum_{t = 0}^{\infty}( h_{t+1} - h_t)_{+} < \infty\]
As $h_{t} \geq 0, \, \forall t \in \Z_{\geq 0}$, the above implies (cf. Lemma~\ref{lem:seq_conv})
\begin{align}
    h_t \underset{t \to \infty}{\longrightarrow} h_{\infty} < \infty \text{ and } \sum_{t = 0}^{\infty} (h_{t+1} - h_t)_{-} > -\infty \label{eqn:conv}
\end{align}
where, operator $(\cdot)_{-}$ is same as defined in Lemma \ref{lem:seq_conv}. As 
\[h_{\infty} - h_0 = \sum_{t = 0}^\infty (h_{t+1}-h_t) \]
Therefore, from \eqref{eqn:final} we get
\[h_{\infty} - h_0 \leq - 2 \sum_{t = 0}^\infty \eta_t \phi_t + (n-f)^2 \mu^2 \Gamma^2 \sum_{t = 0}^\infty  \eta^2_t\]
As $\sum_{t = 0}^\infty \eta^2_t < \infty$, using \eqref{eqn:conv} above implies
\begin{align}
    \sum_{t = 0}^\infty \eta_t \phi_t < \infty \label{eqn:contradict}
\end{align}
Now, we show that $h_{\infty} = 0$ using reasoning by contradiction. Suppose that $h_{\infty} = \beta > 0$, in which case there exists a time $\tau \in Z_{\geq 0}$ such that 
\[\mnorm{h_t - h_\infty} < \beta/2, \quad \forall t \geq \tau\]
This implies, 
\[\beta/2 < h_t = \norm{w^t - w^*}^2 < 3(\beta/2) , \quad  \forall t \geq \tau\]
This implies (refer \eqref{eqn:phi_2}),
\[ \phi_t  > (\lambda n - f (\lambda + 2 \mu)) (\beta/2), \quad  \forall t \geq \tau\]
This implies that if condition \eqref{eqn:cond_1} is satisfied then, 
\[\sum_{t = \tau}^{\infty}\eta_t \phi_t > (\lambda n - f (\lambda + 2 \mu))(\beta/2) \sum_{t = \tau}^{\infty}\eta_t = \infty\]
as $\sum_{t = 0}^{\infty}\eta_t = \infty$ and $\eta_t < \infty, \, \forall t < \infty$. The above is a contradiction of the deduction in \eqref{eqn:contradict}. Therefore, $h_{\infty} \not > 0$ and hence, 
\[ w^t \underset{t \to \infty}{\longrightarrow} w^* \]
\subsection{Proof of Theorem~\ref{thm:mr_2}}
\label{sub:mr_2}
The result is entirely based on the deductions made in the proof of Theorem~\ref{thm:mr_1} (given in Appendix~\ref{sub:mr_1}), except here we need to show that $\phi_t$ (as defined in~\eqref{eqn:phi_t_begin}) is positive when $w^t \neq w^*$ (and $0$ otherwise) under condition~\eqref{eqn:cond_2}, instead of condition~\eqref{eqn:cond_1}. The notation used here is same as in Appendix~\ref{sub:mr_1}. Recall (refer~\eqref{eqn:phi_t_begin}),
\[\phi_t = (n-2f) \cdot \iprod{w^t - w^{*}}{\nabla C_{\H^t_1}(w^t)} - f \cdot \norm{w^t-w^{*}} \cdot \norm{\nabla C_{j_t}(w^t)} \]
Note that due to assumption~\textbf{(A2)}, $\nabla C_{j_t}(w^*) = 0$ and $\nabla C_{\H^t_1}(w^*) = 0$. Therefore, 
\[\phi_t = 0  \text{ if } w^t = w^*\]
From assumption \textbf{(A2)}, $\nabla C_{j_t}(w^*) = 0$ and 
\[ \norm{\nabla C_{j_t}(w^t)} \leq \mu \norm{w^t - w^{*}} \]
Therefore,
\begin{align*}
    \phi_t \geq (n-2f) \cdot \iprod{w^t - w^{*}}{\nabla C_{\H^t_1}(w^t)} - \mu f \cdot \norm{w^t-w^{*}}^2 
\end{align*}
From assumption \textbf{(A5)}, $\iprod{w^t - w^*}{\nabla C_{\H^t_1}(w^t)} \geq \gamma \norm{w^t - w^*}^2$. Thus,
\begin{align*}
    \phi_t \geq (n \gamma - f (2\gamma + \mu)) \norm{w^t - w^*}^2
\end{align*}
Therefore, if condition \eqref{eqn:cond_2} holds, i.e. $(n \gamma - f (2\gamma + \mu)) > 0$, then for any positive real value $\delta$,
\begin{align*}
    \text{ if } \norm{w^t - w^*}^2 > \delta \text{ then } \phi_t > (n \gamma - f (2\gamma + \mu))\delta > 0
\end{align*}
The rest follows immediately from the arguments in the proof of Theorem~\ref{thm:mr_1} (given in Appendix~\ref{sub:mr_1}).

\subsection{Proof of Theorem~\ref{thm:exp_conv}}
\label{sub:exp_conv}
Define $h_t = \norm{w^t - w^{*}}^2$ and let $\eta_t = \eta, \, \forall t \in \Z_{\geq 0}$. From \eqref{eqn:algo_1}, we get
\[ h_{t+1} = \norm{\left[ w^t - \eta \cdot \sum_{\sigma \in \F_t} g^t_{\sigma} \right]_{\W} - w^{*}}^2 \]
Due to the non-expansion property of projection onto a closed convex set~\cite{boyd2004convex}, $\norm{w - w^{*}} \geq \norm{[w]_{\W} - w^{*}}, \, \forall w \in \R^{d}$, therefore 
\[ h_{t+1} \leq \norm{w^t - \eta \cdot \sum_{\sigma \in \F_t} g^t_{\sigma} - w^{*}}^2 = h_t - 2\eta \iprod{w^t - w^{*}}{\sum_{\sigma \in \F_t} g^t_{\sigma}} + \eta^2 \norm{\sum_{\sigma \in \F_t} g^t_{\sigma}}^2\] 
As there are at most $f$ Byzantine agents, thus for each time $t \in \Z_{\geq 0}$ there exists $j_t \in \H$ such that
\[\norm{g^t_{\sigma}} \leq \norm{g^t_{i_{n-f}}} \leq \norm{g^t_{j_t}}, \, \forall \sigma \in \F_t\] 
%using the definition of $\{\overline{g^t_\sigma}\}$ in~\eqref{eqn:filt}.
From assumption \textbf{(A4)}, $g^t_{j_t} = \nabla C_{j_t}(w^{t})$, thus
\begin{align}
	\norm{g^t_{\sigma}} \leq \norm{\nabla C_{j_t}(w^{t})}, \, \forall \sigma \in \F_t \label{eqn:basic_exp}
\end{align}
From assumption~\textbf{(A2)}, we get $\nabla C_{j_t}(w^*) = 0$ and 
\begin{align}
	\norm{\nabla C_{j_t}(w^{t})} \leq \mu \norm{w^t - w^*} \label{eqn:cj_bound_exp}
\end{align}
Thus, 
\[\norm{\sum_{\sigma \in \F_t} g^t_{\sigma}} \leq \sum_{\sigma \in \F_t} \norm{g^t_{\sigma}} \leq \mu (n-f) \norm{w^t - w^*}\]
as $\mnorm{\F_t} = n-f$. This implies,
\begin{align}
h_{t+1} \leq (1 + \mu^2 (n-f)^2 \eta^2) h_{t} - 2\eta \iprod{w^t - w^{*}}{\sum_{\sigma \in \F_t} g^t_{\sigma}} \label{eqn:ht_exp}
\end{align}
As $\mnorm{\F_t} = n-f$ and it is assumed that $f<n/2$ (assumption~\textbf{(A3)}), therefore there exists a subset $\H^t_1 \subset \F_t$ of cardinality $n-2f$ such that $\H^t_1 \subset \H$. Thus, using assumption~\textbf{(A4)} we get
\[g^t_{i} = \nabla C_i(w^t), \, \forall i \in \H^t_1, \, \forall t \in \Z_{\geq 0}.\]
Using the above in \eqref{eqn:ht_exp} gives,
\begin{align}
\begin{split}
h_{t+1} & \leq (1 + \mu^2 (n-f)^2 \eta^2) h_{t} - 2\eta \sum_{i \in \H^t_1}\iprod{w^t - w^{*}}{\nabla C_i(w^t)} - 2\eta \sum_{k \in \F_t \setminus \H^t_1}\iprod{w^t-w^{*}}{g^t_k}\\
& = (1 + \mu^2 (n-f)^2 \eta^2) h_{t} - 2\eta \mnorm{\H^t_1} \iprod{w^t - w^{*}}{\nabla C_{\H^t_1}(w^t)} - 2\eta \sum_{k \in \F_t \setminus \H^t_1}\iprod{w^t-w^{*}}{g^t_k} \label{eqn:ht_split_exp}
\end{split}
\end{align}
where, $C_{\H^t_1} = (1/\mnorm{\H^t_1})\sum_{i \in \H^t_1} C_i$. Using Cauchy-Schwartz inequality,
\begin{align*}
    \mnorm{\iprod{w^t-w^{*}}{g^t_k}} &\leq \norm{w^t-w^{*}} \cdot \norm{g^t_k}, \, \forall k \in \F_t \setminus \H^t_1
\end{align*}
Using \eqref{eqn:basic_exp} and the inequality above above implies (recall $\mnorm{\H^t_1} = n-2f$ and $\mnorm{\F_t} = n-f$),
\begin{align*}
    \sum_{k \in \F_t \setminus \H^t_1}\iprod{w^t-w^{*}}{g^t_k} \geq -f \cdot \norm{w^t-w^{*}} \cdot \norm{\nabla C_{j_t}(w^t)}
\end{align*}
By substituting the above in~\eqref{eqn:ht_split_exp} we obtain,
\begin{align*}
        h_{t+1} \leq (1 + \mu^2 (n-f)^2 \eta^2) h_{t} - 2\eta \left\{ (n-2f) \cdot \iprod{w^t - w^{*}}{\nabla C_{\H^t_1}(w^t)} - f \cdot \norm{w^t-w^{*}} \cdot \norm{\nabla C_{j_t}(w^t)} \right\}
\end{align*}
From~\eqref{eqn:cj_bound_exp},
\[\norm{\nabla C_{j_t}(w^t)} \leq \mu \norm{w^t - w^*}, \, \forall t \in \Z_{\geq 0}\]
Therefore,
\begin{align*}
        h_{t+1} \leq (1 + \mu^2 (n-f)^2 \eta^2) h_{t} - 2\eta \left\{ (n-2f) \cdot \iprod{w^t - w^{*}}{\nabla C_{\H^t_1}(w^t)} - f \mu \norm{w^t-w^{*}}^2 \right\}
\end{align*}
From assumption~\textbf{(A5)},
\[ \iprod{w^t - w^{*}}{\nabla C_{\H^t_1}(w^t)} \geq \gamma \norm{w^t - w^*}^2\]
As $h_t = \norm{w^t - w^*}^2$, thus from above we obtain,
\[h_{t+1} \leq (1 + \mu^2 (n-f)^2 \eta^2) h_{t} - 2\eta \left((n-2f) \gamma  - f \mu\right) h_t\]
or,
\begin{align}
	h_{t+1} \leq \rho^2 h_t, \, \forall t \in \Z_{\geq 0} \label{eqn:ht_final_exp}
\end{align}
where $\rho =  \sqrt{1 - 2\eta(n\gamma - f(2 \gamma + \mu)) + \mu^2 (n-f)^2 \eta^2}$. For
\[\eta = \frac{n\gamma - f(2\gamma + \mu)}{\mu^2 (n-f)^2},\]
which is a positive owing to condition~\eqref{eqn:cond_2}, we get
\[\rho^2 = 1 - \frac{(n\gamma - f(2\gamma + \mu))^2}{\mu^2 (n-f)^2} < 1\]
Now, we verify if the value of $\rho$ above is real, i.e. if the right hand side of the above equality is non-negative. As $\gamma \leq \mu$ (cf. Claim~\ref{clm:mu_gamma}) and condition~\eqref{eqn:cond_2} holds, thus $n > 3f$. This implies that 
\[\frac{(n\gamma - f(2\gamma + \mu))^2}{\mu^2 (n-f)^2} \leq \frac{(n-3f)^2}{(n-f)^2} \leq 1\]
Therefore, 
\[ 1- \frac{(n\gamma - f(2\gamma + \mu))^2}{\mu^2 (n-f)^2} \geq 0 \]
Thus, the value of $\rho$ given above is real and less than one. 
% Further, if condition~\eqref{eqn:cond_2} holds then 
% If condition~\eqref{eqn:cond_2} is satisfied then $\rho < 1$ for all
% \[\eta \in \left(0, \, 2\left(\frac{n\gamma - f(2\gamma + \mu)}{\mu^2 (n-f)^2}\right) \right)\]
% where $n\gamma - f(2\gamma + \mu) > 0$ owing to condition~\eqref{eqn:cond_2}. The minimum value of the convergence factor (corresponding to the faster rate of convergence) is achieved when 
% \[\eta = \frac{n\gamma - f(2\gamma + \mu)}{\mu^2 (n-f)^2}\]
Substituting $h_t = \norm{w^t - w^*}^2$ in~\eqref{eqn:ht_final_exp} concludes the proof.

\subsection{Proof of Lemma~\ref{lem:bnd_growth}}
\label{sub:bnd_growth}

From~\eqref{eqn:algo_1},
\[w^t = \left[ w^{t-1} - \eta_{t-1}\sum_{\sigma \in \F_t} g^t_\sigma \right]_{\W}, \, \forall t \in \mathbb{N}\]
Due to the non-expansion property of projection onto a closed convex set~\cite{boyd2004convex}, $\norm{w - v} \geq \norm{[w]_{\W} - v}, \, \forall w \in \R^d, \,\forall v \in \W$. Therefore, we get
\begin{align}
\norm{w^t - w^{t-1}} = \norm{\left[ w^{t-1} - \eta_{t-1}\sum_{\sigma \in \F_t} g^t_\sigma \right]_{\W} - w^{t-1}} \leq \eta_{t-1}\norm{\sum_{\sigma \in \F_t} g^t_\sigma} \leq \eta_{t-1} \sum_{\sigma \in \F_t} \norm{g^t_\sigma} , \, \forall t \in \mathbb{N} \label{eqn:wt_bound_1}
\end{align}
where, the second inequality follows from the triangle inequality. As there are at most $f$ Byzantine agents, thus for each time $t \in \Z_{\geq 0}$ there exists $j_t \in \H$ such that
\[\norm{g^t_{\sigma}} \leq \norm{g^t_{i_{n-f}}} \leq \norm{g^t_{j_t}}, \, \forall \sigma \in \F_t\] 
using the definition of $\F_t$ in~\eqref{eqn:filter}. From assumption \textbf{(A6)}, $g^t_{j_t} = \nabla C_{j_t}(w^{t - s_{j_t}(t)})$ if $t - s_{j_t}(t) \geq 0$ else $0$, therefore
\begin{align}
	\norm{g^t_{\sigma}} \leq \norm{\nabla C_{j_t}(w^{t - s_{j_t}(t)})}, \, \forall \sigma \in \F_t \label{eqn:cj_bound_lem}
\end{align}
As $w^*$ is assumed to be a minimizer of all the honest agents' cost (cf. assumption~\textbf{(A2)}), 
\[\nabla C_{j_t}(w^*) = 0\]
Therefore, assumption~\textbf{(A2)} implies,
\begin{align*}
	\norm{\nabla C_{j_t}(w)} \leq \mu \norm{w - w^*}, \, \forall w \in \W 
\end{align*}
Let $\Gamma = \max_{w \in \W}\norm{w - w^*}$, where $\W$ is a compact set in $\R^d$. It should be noted that $\Gamma < \infty$ as compact sets in real spaces are bounded. Thus,
\begin{align*}
	\norm{\nabla C_{j_t}(w)} \leq \mu \max_{w \in \W}\norm{w - w^*} = \mu \Gamma, \, \forall w \in \W 
\end{align*}
Substituting the above inequality in~\eqref{eqn:cj_bound_lem} implies that 
\[\norm{g^t_{\sigma}} \leq \mu\Gamma, \, \forall \sigma \in \F_t, \, \forall t \in \Z_{\geq 0}\]
Using this in~\eqref{eqn:wt_bound_1} implies, (recall $\mnorm{\F_t} = n-f$)
\begin{align}
    \norm{w^t - w^{t-1}} \leq \eta_{t-1} \sum_{\sigma \in \F_t} \norm{g^t_\sigma} \leq \eta_{t-1}(n-f)\mu \Gamma, \, \forall t \in \mathbb{N} \label{eqn:delta_w}
\end{align}
From triangle inequality and the fact that $s_i(t) \leq t_o, \, \forall i \in \H$ in assumption~\textbf{(A6)}, we get
\begin{align*}
    \norm{w^t - w^{t - s_i(t)}} \leq \left\{\begin{array}{ccc} \sum_{k = 0}^{t_o-1} \norm{w^{t-k} - w^{t-k-1}} & , & t \geq t_o \\ \empty \\ \sum_{k = 0}^{t-1} \norm{w^{t-k} - w^{t-k-1}} & , & 1 \leq t < t_o\end{array}\right., \, \forall i \in \H
\end{align*}
where, $t_o \geq 1$ as per assumption~\textbf{(A6)}. Using~\eqref{eqn:delta_w} above implies that
\begin{align*}
    \norm{w^t - w^{t - s_i(t)}} \leq \left\{\begin{array}{ccc} (n-f) \mu \Gamma \sum_{k = 0}^{t_o-1} \eta_{t-k-1} & , & t \geq t_o \\ \empty \\ (n-f) \mu \Gamma \sum_{k = 0}^{t-1} \eta_{t-k-1} & , & 1 \leq t < t_o\end{array}\right., \, \forall i \in \H
\end{align*}
Thus, if $\eta_{t+1} \leq \eta_t, \, \forall t \in \Z_{\geq 0}$ then 
\begin{align*}
    \norm{w^t - w^{t - s_i(t)}} \leq \left\{\begin{array}{ccc} \eta_{t - t_o}t_o(n-f)\mu\Gamma & , & t \geq t_o \\ \eta_{0} t (n-f) \mu \Gamma & , & 0 \leq t < t_o\end{array}\right., \, \forall i \in \H
\end{align*}
Therefore,
\begin{align*}
    \sum_{t = 0}^\infty \eta_t \norm{w^t - w^{t-s_i(t)}} & = \sum_{t = t_o}^\infty \eta_t \norm{w^t - w^{t-s_i(t)}} + \sum_{t = 0}^{t_o -1} \eta_t \norm{w^t - w^{t-s_i(t)}} \\
    & \leq \left( t_o\sum_{t = t_o}^\infty \eta_t \eta_{t-t_o} + \eta_0\sum_{t = 0}^{t_o -1} t \eta_t \right) (n-f) \mu \Gamma, \, \forall i \in \H
\end{align*}
As $t_o < \infty$ (assumption~\textbf{(A6)}) and $\eta_t < \infty, \, \forall t < \infty$, thus $\sum_{t = 0}^{t_o -1} t \eta_t < \infty$. As $\eta_{t+1} \leq \eta_t, \, \forall t \in \Z_{\geq 0}$, 
%and $t_o > 0$, 
\[\sum_{t = t_o}^\infty \eta_t \eta_{t-t_o} \leq \sum_{t = t_o}^\infty \eta_{t-t_o} \eta_{t-t_o} = \sum_{t = 0}^\infty \eta^2_t\]
Thus, if $\sum_{t = 0}^\infty \eta^2_t < \infty$ then from above $\sum_{t = t_o}^\infty \eta_t \eta_{t-t_o} < \infty$. Hence,
\[\sum_{t = 0}^\infty \eta_t \norm{w^t - w^{t-s_i(t)}} \leq \left( t_o \sum_{t = 0}^\infty \eta^2_t  + \sum_{t = 0}^{t_o -1} t \eta_t \right) (n-f)\mu \Gamma < \infty, \quad \forall i \in \H\]

\subsection{Proof of Theorem~\ref{thm:mr_delay}}
\label{sub:mr_delay}

Define $h_t = \norm{w^t - w^{*}}^2$. From \eqref{eqn:algo_1}, we get
\[ h_{t+1} = \norm{\left[ w^t - \eta_{t} \cdot \sum_{\sigma \in \F_t}^n g^t_{\sigma} \right]_{\W} - w^{*}}^2 \]
Due to the non-expansion property of projection onto a closed convex set~\cite{boyd2004convex}, $\norm{w - w^{*}} \geq \norm{[w]_{\W} - w^{*}}, \, \forall w \in \R^{d}$, therefore 
\begin{align}
	h_{t+1} \leq \norm{w^t - \eta_{t} \cdot \sum_{\sigma \in \F_t}^n g^t_{\sigma} - w^{*}}^2 = h_t - 2\eta_t \iprod{w^t - w^{*}}{\sum_{\sigma \in \F_t}^n g^t_{\sigma}} + \eta^2_t \norm{\sum_{\sigma \in \F_t}^n g^t_{\sigma}}^2 \label{eqn:ht_bound_delay}
\end{align}
As there are at most $f$ Byzantine agents, thus for each time $t \in \Z_{\geq 0}$ there exists $j_t \in \H$ such that 
\[\norm{g^t_{\sigma}} \leq \norm{g^t_{i_{n-f}}} \leq \norm{g^t_{j_t}}, \, \forall \sigma \in \F_t\] 
using the definition of $\F_t$ in~\eqref{eqn:filter}. From assumption \textbf{(A6)}, $g^t_{j_t} = \nabla C_{j_t}(w^{t - s_{j_t}(t)})$, thus
\begin{align}
	\norm{g^t_{\sigma}} \leq \norm{\nabla C_{j_t}(w^{t - s_{j_t}(t)})}, \, \forall \sigma \in \F_t \label{eqn:cj_delay}
\end{align}
From assumption~\textbf{(A2)}, $\nabla C_{j_t}(w^*) = 0$ and  
\begin{align}
	\norm{\nabla C_{j_t}(w)} \leq \mu \norm{w - w^*} \leq \mu \max_{w \in \W}\norm{w - w^*} = \mu \Gamma, \, \forall w \in \W \label{eqn:cj_delay_2}
\end{align}
where, $\Gamma = \max_{w \in \W}\norm{w - w^*}$ (the right hand side exists due to the fact that $\W$ is compact). Note that $\Gamma < \infty$ as $\W$ is bounded (compact sets in real spaces are closed and bounded). Thus, from \eqref{eqn:ht_bound_delay},~\eqref{eqn:cj_delay} and~\eqref{eqn:cj_delay_2} we get
%\footnote{Here, we use the fact that $w^t \in \W, \, \forall t \in \Z_{\geq 0}$ due to the projection mapping in~\eqref{eqn:algo_1}.}
\begin{align}
h_{t+1} \leq h_t - 2\eta_t \iprod{w^t - w^{*}}{\sum_{\sigma \in \F_t}^n g^t_{\sigma}} + \eta^2_t (n-f)^2 \mu^2 \Gamma^2 \label{eqn:ht_simple_bound}
\end{align}
as from triangle inequality, 
\[\norm{\sum_{\sigma \in \F_t}^n g^t_{\sigma}} \leq  \sum_{\sigma \in \F_t}^n\norm{g^t_{\sigma}} \leq (n-f) \mu \Gamma\]
% Let $\F_t = \{w^0,\ldots,w^t,\sigma_0,\ldots, \sigma_{t-1}\}$. Then,
% \begin{align*}
% \E(h_{t+1}|\F_t) \leq h_t - 2\eta_t \E(\iprod{w^t - w^{*}}{\overline{g^t_{\sigma_t}}}|\F_t) + \eta^2_t \Gamma^2
% \end{align*}
As $\mnorm{\F_t} = n-f$ and it is assumed that $f<n/2$ (assumption~\textbf{(A3)}), therefore there exists a subset $\H^t_1 \subset \F_t$ of cardinality $n-2f$ such that $\H^t_1 \subset \H$. Thus, from assumption~\textbf{(A6)},
\[g^t_{i} = \nabla C_i(w^{t - s_i(t)}), \, \forall i \in \H^t_1, \, \forall t \in \Z_{\geq 0}.\]
Therefore, 
\[\sum_{\sigma \in \F_t}g^t_\sigma = \sum_{i \in \H^t_1} \nabla C_i(w^{t - s_i(t)}) + \sum_{k \in \F_t \setminus \H^t_1} g^t_k , \, \forall t \in \Z_{\geq 0}\]
Substituting this in~\eqref{eqn:ht_simple_bound} gives,
\begin{align*}
&h_{t+1} \leq h_t - 2\eta_t \sum_{i \in \H^t_1}\iprod{w^t - w^{*}}{\nabla C_i(w^{t - s_i(t)})} - 2\eta_t \sum_{k \in \F_t \setminus \H^t_1}\iprod{w^t-w^{*}}{g^t_k} + \eta^2_t (n-f)^2 \mu^2 \Gamma^2
\end{align*}
%$C_{\H^t_1} = (1/\mnorm{\H^t_1})\sum_{i \in \H^t_1} C_i$. \\
Using Cauchy-Schwartz inequality, we get
\begin{align*}
    \mnorm{\iprod{w^t-w^{*}}{g^t_k}} \leq \norm{w^t-w^{*}} \cdot \norm{g^t_k} \leq \norm{w^t-w^{*}} \cdot \norm{\nabla C_{j_t}(w^{t - s_{j_t}(t)})}, \, \forall k \in \F_t \setminus \H^t_1
\end{align*}
Thus (note that $\mnorm{\F_t\setminus \H^t_1} = f$),
\begin{align*}
    \sum_{k \in \F_t \setminus \H^t_1}\iprod{w^t-w^{*}}{g^t_k} \geq -f \cdot \norm{w^t-w^{*}} \cdot \norm{\nabla C_{j_t}(w^{t - s_{j_t}(t)})}
\end{align*}
Therefore,
\begin{align}
	h_{t+1} \leq h_t - 2\eta_t \sum_{i \in \H^t_1}\iprod{w^t - w^{*}}{\nabla C_i(w^{t - s_i(t)})} + 2\eta_t f \cdot \norm{w^t-w^{*}} \cdot \norm{\nabla C_{j_t}(w^{t - s_{j_t}(t)})} + \eta^2_t (n-f)^2 \mu^2 \Gamma^2 \label{eqn:ht_0-0}
\end{align}
By substituting 
\[\nabla C_i(w^{t - s_i(t)}) = \nabla C_i(w^t) + \nabla C_i(w^{t - s_i(t)}) - \nabla C_i(w^t), \, \forall i \in \H^t_1\]
and 
\[\nabla C_{\H^t_1}(w^{t}) = (1/\mnorm{\H^t_1})\sum_{i \in \H^t_1}\nabla C_i (w^t)\]
in~\eqref{eqn:ht_0-0}, we obtain (recall $\mnorm{\H^t_1} = n-2f$),
\begin{align*}
h_{t+1} & \leq h_t - 2\eta_t (n-2f) \iprod{w^t - w^{*}}{\nabla C_{\H^t_1}(w^{t})} - 2\eta_t \sum_{i \in \H^t_1}\iprod{w^t - w^{*}}{\nabla C_i(w^{t - s_i(t)}) - \nabla C_{i}(w^{t})} \\
& + 2\eta_t f \cdot \norm{w^t-w^{*}} \cdot \norm{\nabla C_{j_t}(w^{t - s_{j_t}(t)})} + \eta^2_t (n-f)^2 \mu^2 \Gamma^2
\end{align*}
From triangle inequality, 
\[\norm{\nabla C_{j_t}(w^{t - s_i(t)})} \leq \norm{\nabla C_{j_t}(w^t)} + \norm{\nabla C_{j_t}(w^{t - s_i(t)}) - \nabla C_{j_t}(w^t)}\]
Therefore,
\begin{align*}
h_{t+1} & \leq  h_t - 2\eta_t (n-2f) \cdot \iprod{w^t - w^{*}}{\nabla C_{\H^t_1}(w^{t})} + 2\eta_t f \cdot \norm{w^t-w^{*}} \cdot \norm{\nabla C_{j_t}(w^t)} \\ 
& - 2\eta_t \sum_{i \in \H^t_1}\iprod{w^t - w^{*}}{\nabla C_i(w^{t - s_i(t)}) - \nabla C_i(w^{t})} + 2\eta_t f \cdot \norm{w^t-w^{*}} \cdot \norm{\nabla C_{j_t}(w^{t - s_{j_t}(t)}) - \nabla C_{j_t}(w^t)}\\
&+ \eta^2_t (n-f)^2 \mu^2 \Gamma^2
\end{align*}
If we let 
\begin{align}
\phi_t = (n-2f) \cdot \iprod{w^t - w^{*}}{\nabla C_{\H^t_1}(w^t)} - f \cdot \norm{w^t-w^{*}} \cdot \norm{\nabla C_{j_t}(w^t)} \label{eqn:phi_t_delay}
\end{align}
then the above inequality becomes
\begin{align*}
    h_{t+1} & \leq h_t - 2\eta_t \phi_t + \eta^2_t (n-f)^2 \mu^2 \Gamma^2 \\
    & - 2\eta_t \sum_{i \in \H^t_1}\iprod{w^t - w^{*}}{\nabla C_i(w^{t - s_i(t)}) - \nabla C_i(w^{t})} + 2\eta_t f \cdot \norm{w^t-w^{*}} \cdot \norm{\nabla C_{j_t}(w^{t - s_{j_t}(t)}) - \nabla C_{j_t}(w^t)}
\end{align*}
From Cauchy-Schwartz inequality, 
\[\mnorm{\iprod{w^t - w^{*}}{\nabla C_i(w^{t - s_i(t)}) - \nabla C_i(w^{t})} } \leq \norm{w^t - w^{*}} \cdot \norm{\nabla C_i(w^{t - s_i(t)}) - \nabla C_i(w^{t})}, \, \forall i \in \H^t_1\]
Therefore,
\begin{align*}
    h_{t+1} & \leq h_t - 2\eta_t \phi_t + \eta^2_t (n-f)^2 \mu^2 \Gamma^2 \\
    & + 2\eta_t \sum_{i \in \H^t_1}\norm{w^t - w^{*}} \cdot \norm{\nabla C_i(w^{t - s_i(t)}) - \nabla C_i(w^{t})} + 2\eta_t f \cdot \norm{w^t-w^{*}} \cdot \norm{\nabla C_{j_t}(w^{t - s_{j_t}(t)}) - \nabla C_{j_t}(w^t)}
\end{align*}
From assumption~\textbf{(A2)}, 
\begin{align*}
    \norm{\nabla C_i(w^{t - s_i(t)}) - \nabla C_i(w^{t})} & \leq \mu \norm{w^{t - s_i(t)} - w^t}, \, \forall i \in \H^t_1 \text{, and } \\
    \norm{\nabla C_{j_t}(w^{t - s_{j_t}(t)}) - \nabla C_{j_t}(w^t)} & \leq \mu \norm{w^{t - s_{j_t}(t)} - w^t}
\end{align*}
This implies,
\begin{align*}
    h_{t+1} & \leq h_t - 2\eta_t \phi_t + \eta^2_t (n-f)^2 \mu^2 \Gamma^2 \\
    & + 2\eta_t \mu \sum_{i \in \H^t_1}\norm{w^t - w^{*}} \cdot \norm{w^{t - s_i(t)} - w^t} + 2\eta_t f \mu \cdot \norm{w^t-w^{*}} \cdot \norm{w^{t - s_{j_t}(t)} - w^t}
\end{align*}
As mentioned earlier, $\norm{w^t - w^*} \leq \Gamma = \max_{w \in \W}\norm{w - w^*}, \, \forall t \in \Z_{\geq 0}$ owing to the fact that $w^t \in \W,  \, \forall t \in \Z_{\geq 0}$ (refer~\eqref{eqn:algo_1}) and $\W$ is a compact set. This implies
\begin{align}
        h_{t+1} \leq h_t - 2\eta_t \phi_t + \eta^2_t (n-f)^2 \mu^2 \Gamma^2 + 2\eta_t \left(\sum_{i \in \H^t_1} \norm{w^{t - s_i(t)} - w^t} \right) \mu \Gamma + 2\eta_t \norm{w^{t - s_{j_t}(t)} - w^t} f \mu \Gamma \label{eqn:d_ineq_1}
\end{align}
Now, we establish that $\phi_t$ (as defined in~\eqref{eqn:phi_t_delay}) is non-negative $\forall t \in \Z_{\geq 0}$ as follows, under condition~\eqref{eqn:cond_2} by considering two possible cases; case (i) $w^t = w^*$, and case (ii) $w^t \neq w^*$.\\
Case (i) If $w^t = w^*$, then from assumption~\textbf{(A2)}, $\nabla C_{j_t}(w^t) = 0$ and $\nabla C_{\H^t_1}(w^t) = 0$. Therefore (ref.~\eqref{eqn:phi_t_delay}), 
\begin{align}
    \phi_t = 0 \text{ if } w^t = w^* \label{eqn:d_phi_0}
\end{align}
Case (ii) Let $w^t \neq w^*$. From assumption \textbf{(A2)}, $\nabla C_{j_t}(w^*) = 0$ and
\[ \norm{\nabla C_{j_t}(w^t)} \leq \mu \norm{w^t - w^{*}} \]
Therefore (ref.~\eqref{eqn:phi_t_delay}),
\begin{align}
    \phi_t \geq (n-2f) \cdot \iprod{w^t - w^{*}}{\nabla C_{\H^t_1}(w^t)} - \mu f \cdot \norm{w^t-w^{*}}^2 \label{eqn:d_phi_1}
\end{align}
From assumption \textbf{(A5)}, $\iprod{w^t - w^*}{\nabla C_{\H^t_1}(w^t)} \geq \gamma \norm{w^t - w^*}^2$. Thus,
\begin{align*}
    \phi_t \geq (n \gamma - f (2\gamma + \mu)) \norm{w^t - w^*}^2
\end{align*}
Therefore, condition \eqref{eqn:cond_2} implies that for any positive real value $\delta$,
\begin{align}
    \text{ if } \norm{w^t - w^*}^2 > \delta \text{ then } \phi_t > (n \gamma - f (2\gamma + \mu))\delta > 0 \label{eqn:d_phi_2}
\end{align}
Owing to \eqref{eqn:cond_2}, \eqref{eqn:d_ineq_1}, \eqref{eqn:d_phi_0} and \eqref{eqn:d_phi_2}, we get
\begin{align}\label{eqn:d_final}
\begin{split}
        h_{t+1} & \leq h_t - 2\eta_t \phi_t + \eta^2_t (n-f)^2 \mu^2 \Gamma^2 + 2\eta_t \left(\sum_{i \in \H^t_1} \norm{w^{t - s_i(t)} - w^t} \right) \mu \Gamma + 2\eta_t \norm{w^{t - s_{j_t}(t)} - w^t} f \mu \Gamma  \\
        & \leq h_t + \eta^2_t (n-f)^2 \mu^2 \Gamma^2 + 2\eta_t \left(\sum_{i \in \H^t_1} \norm{w^{t - s_i(t)} - w^t} \right) \mu \Gamma + 2\eta_t \norm{w^{t - s_{j_t}(t)} - w^t} f \mu \Gamma
\end{split}
\end{align}
Therefore,
\begin{align*}
    (h_{t+1} - h_t)_{+} \leq \eta^2_t (n-f)^2 \mu^2 \Gamma^2 + 2\eta_t \left(\sum_{i \in \H^t_1} \norm{w^{t - s_i(t)} - w^t} \right) \mu \Gamma + 2\eta_t \norm{w^{t - s_{j_t}(t)} - w^t} f \mu \Gamma, \, \forall t \in \Z_{\geq 0}
\end{align*}
where, operator $(\cdot)_{+}$ is same as defined in Lemma \ref{lem:seq_conv}. As $\eta_{t+1} \leq \eta_t, \, \forall t \in Z_{\geq 0}$, $\sum_{t = 0}^{\infty}\eta^2_t < \infty$ and $\Gamma < \infty$, therefore from Lemma~\ref{lem:bnd_growth}, we get
\begin{align} \label{eqn:d_bnd_sum}
\begin{split}
    \sum_{t = 0}^\infty(h_{t+1} - h_t)_{+} & \leq (n-f)^2 \mu^2 \Gamma^2 \sum_{t = 0}^\infty \eta^2_t  + 2 \mu \Gamma \sum_{i \in \H^t_1} \sum_{t = 0}^\infty \eta_t \norm{w^{t - s_i(t)} - w^t}  + 2 f \mu \Gamma \sum_{t = 0}^\infty \eta_t \norm{w^{t - s_{j_t}(t)} - w^t}\\
    & < \infty
\end{split}
\end{align}
As $h_{t} \geq 0, \, \forall t \in \Z_{\geq 0}$, the above implies (cf. Lemma~\ref{lem:seq_conv})
\begin{align}
    h_t \underset{t \to \infty}{\longrightarrow} h_{\infty} < \infty \text{ and } \sum_{t = 0}^{\infty} (h_{t+1} - h_t)_{-} > -\infty \label{eqn:d_conv}
\end{align}
where, operator $(\cdot)_{-}$ is same as defined in Lemma \ref{lem:seq_conv}. As 
\[h_{\infty} - h_0 = \sum_{t = 0}^\infty (h_{t+1}-h_t) \]
Therefore, from \eqref{eqn:d_final} we get
%\rule{\textwidth}{0.4pt}
\begin{align*}
   h_{\infty} - h_0 \leq -2 \sum_{t = 0}^\infty \eta_t \phi_t + (n-f)^2 \mu^2 \Gamma^2 \sum_{t = 0}^\infty\eta^2_t  + 2 \mu \Gamma \sum_{i \in \H^t_1} \sum_{t = 0}^\infty \eta_t \norm{w^{t - s_i(t)} - w^t} + 2 f \mu \Gamma \sum_{t = 0}^\infty \eta_t \norm{w^{t - s_{j_t}(t)} - w^t}  
\end{align*}
Using~\eqref{eqn:d_bnd_sum} and~\eqref{eqn:d_conv} above implies
\begin{align}
    \sum_{t = 0}^\infty \eta_t \phi_t < \infty \label{eqn:contradict_2}
\end{align}
Now, we show that $h_{\infty} = 0$ using reasoning by contradiction. Suppose that $h_{\infty} = \beta > 0$, in which case there exists a time $\tau \in Z_{\geq 0}$ such that 
\[\mnorm{h_t - h_\infty} < \beta/2, \quad \forall t \geq \tau\]
This implies, 
\[\beta/2 < h_t = \norm{w^t - w^*}^2 < 3(\beta/2) , \quad  \forall t \geq \tau\]
This implies (refer \eqref{eqn:d_phi_2}),
\[ \phi_t > (n \gamma - f (2\gamma + \mu)) \cdot (\beta/2), \quad  \forall t \geq \tau\]
This implies that if condition \eqref{eqn:cond_1} is satisfied then, 
\[\sum_{t = \tau}^{\infty}\eta_t \phi_t > (n \gamma - f (2\gamma + \mu)) \cdot (\beta/2) \sum_{t = \tau}^{\infty}\eta_t = \infty\]
as $\sum_{t = 0}^{\infty}\eta_t = \infty$ and $\eta_t < \infty, \, \forall t < \infty$. The above is a contradiction of the deduction in \eqref{eqn:contradict_2}. Therefore, $h_{\infty} \not > 0$ and hence, 
\[ w^t \underset{t \to \infty}{\longrightarrow} w^* \]

\subsection{Proof of Theorem~\ref{thm:mr_noise}}
\label{app:mr_noise}

%\noindent \textbf{Note:} The proof borrows ideas from `global confinement' analysis in Bottou, 1998~\cite{bottou1998online}.\\
Define a scalar function $\psi: \R \to \R$ as follows (cf. Bottou, 1998~\cite{bottou1998online}):
\begin{align}
\psi(x) = \left\{\begin{array}{ccc}0 &, & x \leq \widehat{D} \\ \left(x - \widehat{D}\right)^2 & , & \text{o.w.} \end{array} \right. \label{eqn:def_psi}
\end{align}
where, 
\[\widehat{D} = \left(D^*\right)^2 = \left( \frac{n-2f}{n\gamma - f (2\gamma + \mu)} D \right)^2 \]
Note that (cf. Bottou, 1998~\cite{bottou1998online}) 
\begin{align}
    \psi(y) - \psi(x) \leq (y-x)\psi' (x) + (y-x)^2 \label{eqn:psi_bnd}, \, \forall x,\, y \in \R_{\geq 0}
\end{align}
where, $\psi^{'}(x)$ is the derivative of $\psi$ at $x$. Define, 
\begin{align}
    h_t = \psi \left( \norm{w^t - w^*}^2\right) \label{eqn:ht_def_noise}
\end{align}
Note that $h_t = 0$ iff $\norm{w^t - w^*} \leq D^*$. Inequality~\eqref{eqn:psi_bnd} implies that
\begin{align*}
    h_{t+1} - h_t & = \psi \left( \norm{w^{t+1} - w^*}^2\right) - \psi \left( \norm{w^t - w^*}^2\right) \\
    & \leq \left(\norm{w^{t+1} - w^*}^2 - \norm{w^t - w^*}^2\right) \cdot \psi' \left( \norm{w^t - w^*}^2\right) + \left( \norm{w^{t+1} - w^*}^2 - \norm{w^t - w^*}^2 \right)^2
\end{align*}
For the sake of convenience, let
\[\psi'_t \triangleq \psi' \left( \norm{w^t - w^*}^2\right)\]
Note that
\begin{align}
    \psi'_t = \left\{\begin{array}{ccc} 0 & , & \norm{w^t - w^*}^2 \leq \widehat{D} \\ 2 \left(\norm{w^t - w^*}^2 - \widehat{D}\right) & , & \norm{w^t - w^*}^2 > \widehat{D} \end{array}\right. \label{eqn:psi_t}
\end{align}
We can re-write the above inequality as,
\begin{align}
    h_{t+1} - h_t & \leq \left(\norm{w^{t+1} - w^*}^2 - \norm{w^t - w^*}^2\right)  \psi'_t + \left( \norm{w^{t+1} - w^*}^2 - \norm{w^t - w^*}^2 \right)^2 , \, \forall t \in \Z_{\geq 0} \label{eqn:ht_1_noise}
\end{align}
From~\eqref{eqn:algo_1}, we know that 
\[w^{t+1} = \left[ w^t - \eta_{t} \cdot \sum_{\sigma \in \F_t} g^t_{\sigma} \right]_{\W} \]
From the non-expansion property of the projection onto a closed convex set~\cite{boyd2004convex}, $\norm{w - w^{*}} \geq \norm{[w]_{\W} - w^{*}}, \, \forall w \in \R^{d}$. Therefore,
\[\norm{w^{t+1} - w^*} \leq \norm{w^t - w^* - \eta_{t} \cdot \sum_{\sigma \in \F_t} g^t_{\sigma}}\]
\begin{align}
    \implies \norm{w^{t+1} - w^*}^2 \leq \norm{w^t - w^*}^2  - 2 \eta_{t} \iprod{w^t - w^*}{\sum_{\sigma \in \F_t} g^t_{\sigma}} +  \eta_t^2 \norm{\sum_{\sigma \in \F_t} g^t_{\sigma}}^2 \label{eqn:proj_bnd}
\end{align}
As $\psi'_t \geq 0, \, \forall t \in \Z_{\geq 0}$ (refer~\eqref{eqn:psi_t}), therefore~\eqref{eqn:ht_1_noise} and~\eqref{eqn:proj_bnd} implies that
\begin{align}
    h_{t+1} - h_t &\leq \left(- 2 \eta_{t} \iprod{w^t - w^*}{\sum_{\sigma \in \F_t} g^t_{\sigma}} + \eta^2_t \norm{\sum_{\sigma \in \F_t} g^t_{\sigma}}^2 \right)\psi'_t  +  \left( \norm{w^{t+1} - w^*}^2 - \norm{w^t - w^*}^2 \right)^2 \label{eqn:ht_2_noise}
\end{align}
Note that,
\begin{align*}
	\mnorm{\norm{w^{t+1} - w^*}^2 - \norm{w^t - w^*}^2}  = \left(\norm{w^{t+1} - w^*} + \norm{w^t - w^*}\right)  \mnorm{ \norm{w^{t+1} - w^*} - \norm{w^t - w^*}} 
\end{align*}
As $w^t \in \W \subset \R^d, \, \forall t \in \Z_{\geq 0}$, where $\W$ is a compact set (closed and bounded in $\R^d$), therefore 
\begin{align}
     \norm{w^t - w^*} \leq \Gamma = \max_{w \in \W}\norm{w - w^*} < \infty, \quad \forall t \in \Z_{\geq 0}    \label{eqn:wt_bnd}
\end{align}
Let $\Gamma > 0$ (otherwise, $\W$ contains only $w^*$, and the problem is trivial). Thus, \[\norm{w^{t+1} - w^*} + \norm{w^{t} - w^*} \leq 2\Gamma, \,\forall t \in \Z_{\geq 0}\] Therefore,
\begin{align}
	\mnorm{\norm{w^{t+1} - w^*}^2 - \norm{w^t - w^*}^2}  \leq 2\Gamma \mnorm{ \norm{w^{t+1} - w^*} - \norm{w^t - w^*}} \label{eqn:ab}
\end{align}
From triangle inequality 
\begin{align}
	\norm{w^{t+1} - w^*} - \norm{w^t - w^*} \leq \norm{w^{t+1} - w^t} \label{eqn:ab_t1}
\end{align}
and,
\begin{align}
	\norm{w^{t+1} - w^*} - \norm{w^t - w^*} \geq -\norm{w^{t+1} - w^t} \label{eqn:ab_t2}
\end{align}
Inequalities~\eqref{eqn:ab_t1} and~\eqref{eqn:ab_t2} imply that
\begin{align}
	\mnorm{ \norm{w^{t+1} - w^*} - \norm{w^t - w^*}} \leq \norm{w^{t+1} - w^t} \label{eqn:ab_t3}
\end{align}
Substituting~\eqref{eqn:ab_t3} in~\eqref{eqn:ab} implies that 
\begin{align}
	\mnorm{\norm{w^{t+1} - w^*}^2 - \norm{w^t - w^*}^2}  \leq 2\Gamma \norm{w^{t+1} - w^t} \label{eqn:ab_2}
\end{align}
Using the non-expansion property of the projection onto a closed convex set, 
\[\norm{w^{t+1} - w^t} = \norm{ \left[ w^{t} -\eta_t \sum_{\sigma \in \F_t}g^t_\sigma \right]_{\W} - w^t} \leq \eta_t\norm{\sum_{\sigma \in \F_t}g^t_\sigma}\]
Substituting this in~\eqref{eqn:ab_2} implies that
\begin{align*}
	\mnorm{\norm{w^{t+1} - w^*}^2 - \norm{w^t - w^*}^2}  \leq 2 \eta_t \Gamma \norm{\sum_{\sigma \in \F_t}g^t_\sigma} 
\end{align*}
\begin{align}
	\implies \left(\norm{w^{t+1} - w^*}^2 - \norm{w^t - w^*}^2\right)^2  \leq 4 \eta^2_t \Gamma^2 \norm{\sum_{\sigma \in \F_t}g^t_\sigma}^2 \label{eqn:ab_f}
\end{align}
%From Cauchy-Schwartz inequality, 
%\[\mnorm{\iprod{w^t - w^*}{\sum_{\sigma \in \F_t} g^t_{\sigma}}} \leq \norm{w^t - w^*} \cdot \norm{\sum_{\sigma \in \F_t} g^t_{\sigma}}\]
%Thus,~\eqref{eqn:proj_bnd} implies that
%\begin{align}
%\begin{split}
%    \mnorm{\norm{w^{t+1} - w^*}^2 - \norm{w^t - w^*}^2} & \leq   2 \eta_{t} \mnorm{\iprod{w^t - w^*}{\sum_{\sigma \in \F_t} g^t_{\sigma}}} + \eta_t^2 \norm{\sum_{\sigma \in \F_t} g^t_{\sigma}}^2 \\
%    & \leq 2 \eta_t \norm{w^t - w^*} \norm{\sum_{\sigma \in \F_t} g^t_{\sigma}} + \eta_t^2 \norm{\sum_{\sigma \in \F_t} g^t_{\sigma}}^2 \label{eqn:proj_bnd_2}
%\end{split}
%\end{align}
Substituting~\eqref{eqn:ab_f} in~\eqref{eqn:ht_2_noise} implies that
\begin{align*}
    h_{t+1} - h_t &\leq \left(- 2 \eta_{t} \iprod{w^t - w^*}{\sum_{\sigma \in \F_t} g^t_{\sigma}} + \eta^2_t \norm{\sum_{\sigma \in \F_t} g^t_{\sigma}}^2 \right)\psi'_t  + 
    4 \eta^2_t \Gamma^2 \norm{\sum_{\sigma \in \F_t}g^t_\sigma}^2
\end{align*}
\begin{align}
    \implies h_{t+1} - h_t &\leq - 2 \eta_{t} \iprod{w^t - w^*}{\sum_{\sigma \in \F_t} g^t_{\sigma}}\psi'_t + \eta^2_t \left\{ \norm{\sum_{\sigma \in \F_t} g^t_{\sigma}}^2 \psi'_t  + 4  \Gamma^2 \norm{\sum_{\sigma \in \F_t}g^t_\sigma}^2 \right\}, \quad \forall t \in \Z_{\geq 0} \label{eqn:ht_ht_noise}
\end{align}
As $D < \infty$ (cf. Assumption~\textbf{(A7)}),~\eqref{eqn:psi_t} and~\eqref{eqn:wt_bnd} implies that
\begin{align}
   0 \leq \psi'_t \leq 2 (\Gamma^2 - \widehat{D}) \leq 2 \Gamma^2 <  \infty, \quad \forall t \in \Z_{\geq 0} \label{eqn:d_psi_bnd}
\end{align}
As there are at most $f$ Byzantine agents, thus for each time $t \in \Z_{\geq 0}$ there exists $j_t \in \H$ such that
\[\norm{g^t_{\sigma}} \leq \norm{g^t_{i_{n-f}}} \leq \norm{g^t_{j_t}}, \, \forall \sigma \in \F_t\] 
%using the definition of $\{\overline{g^t_\sigma}\}$ in~\eqref{eqn:filt}.
From assumption \textbf{(A4)}, $g^t_{j_t} = \nabla C_{j_t}(w^{t})$, thus
\begin{align}
	\norm{g^t_{\sigma}} \leq \norm{\nabla C_{j_t}(w^{t})}, \, \forall \sigma \in \F_t \label{eqn:c_jt_noise}
\end{align}
As $w^*$ is assumed to be a minimizer of all the honest agents' cost (cf. assumption~\textbf{(A2)}), 
\[\nabla C_{j_t}(w^*) = 0\]
Therefore, assumption~\textbf{(A2)} and~\eqref{eqn:wt_bnd} imply that
\begin{align}
	\norm{\nabla C_{j_t}(w^t)} \leq \mu \norm{w^t - w^*} \leq \mu \Gamma, \, \forall t \in \Z_{\geq 0} \label{eqn:cj_bound_noise}
\end{align}
Using triangle inequality, \eqref{eqn:c_jt_noise} and \eqref{eqn:cj_bound_noise} we obtain (recall, $\mnorm{\F_t} = n-f$),
\begin{align}
    \norm{\sum_{\sigma \in \F_t} g^t_{\sigma}} \leq \sum_{\sigma \in \F_t}\norm{g^t_{\sigma}} \leq (n-f) \mu \Gamma < \infty \label{eqn:gt_bnd_noise}
\end{align}
Using~\eqref{eqn:d_psi_bnd} and~\eqref{eqn:gt_bnd_noise} in~\eqref{eqn:ht_ht_noise} implies
\begin{align*}
    h_{t+1} - h_t &\leq -2\eta_t \iprod{w^t - w^*}{\sum_{\sigma \in \F_t} g^t_{\sigma}}\psi'_t +  \eta^2_t \left\{2(n-f)^2\mu^2\Gamma^4 + 4 (n-f)^2 \mu^2 \Gamma^4 \right\}, \, \forall t \in \Z_{\geq 0}
\end{align*}
Let,
\begin{align}
    K = 2(n-f)^2\mu^2\Gamma^4 + 4 (n-f)^2 \mu^2 \Gamma^4 = 6(n-f)^2 \mu^2 \Gamma^4 \label{eqn:omega}
\end{align}
Then,
\begin{align}
    h_{t+1} - h_t \leq -2\eta_t \iprod{w^t - w^*}{\sum_{\sigma \in \F_t} g^t_{\sigma}}\psi'_t + \eta^2_t K, \quad \forall t \in \Z_{\geq 0} \label{eqn:ht_simple_noise}
\end{align}
As $\mnorm{\F_t} = n-f$ and it is assumed that $f<n/2$ (assumption~\textbf{(A3)}), there exists a subset $\H^t_1 \subset \F_t$ of cardinality $n-2f$ such that $\H^t_1 \subset \H$. Thus, from assumption~\textbf{(A7)} we get
\[g^t_{i} = \nabla C_i(w^t) + D_i(w^t), \, \forall i \in \H^t_1, \, \forall t \in \Z_{\geq 0}.\]
where, $\norm{D_i(w^t)} \leq D < \infty,  \, \forall i \in \H^t_1, \, \forall t \in \Z_{\geq 0}$. Substituting the above in~\eqref{eqn:ht_simple_noise} implies that
\begin{align*}
h_{t+1} & \leq h_t - 2\eta_t \left(\sum_{i \in \H^t_1}\iprod{w^t - w^{*}}{\nabla C_i(w^t)} + \sum_{i \in \H^t_1}\iprod{w^t - w^{*}}{D_i(w^t)} + \sum_{k \in \F_t \setminus \H^t_1}\iprod{w^t-w^{*}}{g^t_k}\right) \psi'_t \\
& + \eta^2_t K 
\end{align*}
Substituting $C_{\H^t_1} = (1/\mnorm{\H^t_1})\sum_{i \in \H^t_1} C_i$ above implies,
\begin{align}
\begin{split}
& h_{t+1} = h_t - 2\eta_t \left(\mnorm{\H^t_1} \iprod{w^t - w^{*}}{\nabla C_{\H^t_1}(w^t)} + \sum_{i \in \H^t_1}\iprod{w^t - w^{*}}{D_i(w^t)} + \sum_{k \in \F_t \setminus \H^t_1}\iprod{w^t-w^{*}}{g^t_k}\right) \psi'_t \\
& + \eta^2_t K, \quad \forall t \in \Z_{\geq 0} \label{eqn:ht_split_noise}
\end{split}
\end{align}
Using Cauchy-Schwartz inequality, we get
\begin{align*}
    \mnorm{\iprod{w^t-w^{*}}{g^t_k}} &\leq \norm{w^t-w^{*}} \cdot \norm{g^t_k}, \, \forall k \in \F_t \setminus \H^t_1
\end{align*}
Using~\eqref{eqn:c_jt_noise} above implies that
\[\mnorm{\iprod{w^t-w^{*}}{g^t_k}} \leq \norm{w^t-w^{*}} \cdot \norm{\nabla C_{j_t}(w^t)},  \, \forall k \in \F_t \setminus \H^t_1\]
Thus,
\[\iprod{w^t-w^{*}}{g^t_k} \geq - \norm{w^t-w^{*}} \cdot \norm{\nabla C_{j_t}(w^t)},  \, \forall k \in \F_t \setminus \H^t_1\]
As $\mnorm{\H^t_1} = n-2f$ and $\mnorm{\F_t} = n-f$, 
\begin{align}
    \sum_{k \in \F_t \setminus \H^t_1}\iprod{w^t-w^{*}}{g^t_k} \geq -f \cdot \norm{w^t-w^{*}} \cdot \norm{\nabla C_{j_t}(w^t)} \label{eqn:bf_bound_noise}
\end{align}
Similarly, from Cauchy-Schwartz inequality, we also get
\begin{align*}
    \mnorm{\iprod{w^t - w^*}{D_i(w^t)}} \leq \norm{w^t - w^*} \norm{D_i(w^t)}, \, \forall i \in \H^t_1
\end{align*}
Using Assumption~\textbf{(A7)} above implies,
\begin{align*}
    \iprod{w^t - w^*}{D_i(w^t)} \geq - D \norm{w^t - w^*}, \, \forall i \in \H^t_1
\end{align*}
Thus,
\begin{align}
    \sum_{i \in \H^t_1}\iprod{w^t - w^*}{D_i(w^t)} \geq -\mnorm{\H^t_1} D \norm{w^t - w^*}, \, \forall t \in \Z_{\geq 0} \label{eqn:noise_grad_bnd}
\end{align}
From substituting~\eqref{eqn:bf_bound_noise} and~\eqref{eqn:noise_grad_bnd} in~\eqref{eqn:ht_split_noise}, we obtain, 
\begin{align*}
        h_{t+1} & \leq h_t - 2\eta_t \left\{ \mnorm{\H^t_1} \iprod{w^t - w^{*}}{\nabla C_{\H^t_1}(w^t)} - \mnorm{\H^t_1} D \norm{w^t - w^*} - f \norm{w^t-w^{*}} \norm{\nabla C_{j_t}(w^t)} \right\}\psi'_t \\
        & + \eta^2_t K, \quad \forall t \in \Z_{\geq 0}
\end{align*}
If we let 
\begin{align}
	\phi_t = \left\{ \mnorm{\H^t_1} \iprod{w^t - w^{*}}{\nabla C_{\H^t_1}(w^t)} - \mnorm{\H^t_1} D \norm{w^t - w^*} - f \norm{w^t-w^{*}} \norm{\nabla C_{j_t}(w^t)}  \right\}\psi'_t \label{eqn:phi_t_noise}
\end{align}
Then the last inequality can be re-written as 
\begin{align}
    h_{t+1} \leq h_t - 2\eta_t \phi_t  + \eta_t^2 K , \quad \forall t \in \Z_{\geq 0} \label{eqn:ineq_1_noise}
\end{align}
Next, we show that $\phi_t \geq 0, \, \forall t \in \Z_{\geq 0}$ by considering two cases, case (i) $\norm{w^t - w^*} \leq \sqrt{\widehat{D}}$ and case (ii) $\norm{w^t - w^*} > \sqrt{\widehat{D}}$, as follows. \\
%Recall,
%\[\sqrt{\widehat{D}} = \frac{n-2f}{n \gamma - f (2 \gamma + \mu)}D\]
\noindent Case (i): From~\eqref{eqn:psi_t}, if $\norm{w^t - w^*}^2 \leq \widehat{D}$ then $\psi'_t = 0$. Thus, 
\begin{align}
    \phi_t = 0, \, \text{ if } \norm{w^t - w^*} \leq \sqrt{\widehat{D}} \label{eqn:case1_noise}
\end{align}
Case (ii): From~\eqref{eqn:psi_t}, if $\norm{w^t - w^*}^2 > \widehat{D}$ then $\psi'_t = 2(\norm{w^t - w^*}^2 - \widehat{D}) > 0$. From~\eqref{eqn:cj_bound_noise}, 
\[- f \norm{w^t-w^{*}} \norm{\nabla C_{j_t}(w^t)} \geq  - f \mu \norm{w^t - w^*}^2\]
From assumption~\textbf{(A5)}, 
\[\iprod{w^t - w^{*}}{\nabla C_{\H^t_1}(w^t)} \geq \gamma \norm{w^t - w^*}^2\]
As $\mnorm{\H^t_1} = n - 2f > 0$, from the last two inequalities we obtain,
\[\mnorm{\H^t_1} \iprod{w^t - w^{*}}{\nabla C_{\H^t_1}(w^t)} - f \norm{w^t-w^{*}} \norm{\nabla C_{j_t}(w^t)} \geq (n \gamma   - f (2 \gamma + \mu)) \norm{w^t - w^*}^2 \]
Substituting the above in~\eqref{eqn:phi_t_noise} implies that (recall, $\psi'_t \geq 0$ and $\mnorm{\H^t_1} = n - 2f$ for all $t \in \Z_{\geq 0}$)
\begin{align}
\phi_t \geq \left\{(n \gamma   - f (2 \gamma + \mu)) \norm{w^t - w^*}^2 - (n-2f)D \norm{w^t - w^*}\right\} \psi'_t \label{eqn:phi_t_case2_noise}
\end{align}
For a positive real value $\delta$, if $\norm{w^t - w^*} \geq \sqrt{\widehat{D}} + \sqrt{\delta}$ then (ref.~\eqref{eqn:psi_t}),
\begin{align}
    \psi'_t = 2 (\norm{w^t - w^*}^2 - \widehat{D}) = 2 \left(\norm{w^t - w^*} + \sqrt{\widehat{D}}\right) \left(\norm{w^t - w^*} - \sqrt{\widehat{D}}\right)\geq 2 \sqrt{\delta} \left((\sqrt{\delta} + 2\sqrt{\widehat{D}}\right) \geq \delta \label{eqn:case2_temp1}
\end{align}
Recall,
\[\sqrt{\widehat{D}} = \frac{n-2f}{n \gamma   - f (2 \gamma + \mu)} D \, \implies (n \gamma   - f (2 \gamma + \mu))\sqrt{\widehat{D}} = (n-2f) D\]
Thus, if $\norm{w^t - w^*} \geq \sqrt{\widehat{D}} + \sqrt{\delta}$ then in~\eqref{eqn:phi_t_case2_noise},
\[(n \gamma   - f (2 \gamma + \mu)) \norm{w^t - w^*} - (n-2f)D \geq (n \gamma   - f (2 \gamma + \mu)) \left( \sqrt{\widehat{D}} + \sqrt{\delta} \right) - (n-2f)D \geq (n \gamma   - f (2 \gamma + \mu)) \sqrt{\delta}\]
\begin{align}
\begin{split}
	\implies & (n \gamma   - f (2 \gamma + \mu)) \norm{w^t - w^*}^2 - (n-2f)D \norm{w^t - w^*} \\
	&= \norm{w^t - w^*} \left((n \gamma   - f (2 \gamma + \mu)) \norm{w^t - w^*} - (n-2f)D\right) \geq (n \gamma   - f (2 \gamma + \mu)) \sqrt{\delta}\left(\sqrt{\widehat{D}} + \sqrt{\delta}\right) \\ \label{eqn:case2_temp2}
	&\geq (n \gamma   - f (2 \gamma + \mu)) \delta
\end{split}
\end{align}
Since condition~\eqref{eqn:cond_2} is assumed to hold (i.e. $n \gamma  > f (2 \gamma + \mu)$), from~\eqref{eqn:phi_t_case2_noise},~\eqref{eqn:case2_temp1} and~\eqref{eqn:case2_temp2} we infer the following.
\begin{align}
    \text{If } \norm{w^t - w^*} \geq \sqrt{\widehat{D}} + \sqrt{\delta} \text{, for $\delta \in \R_{\geq 0}$, then } \phi_t \geq  \left(n \gamma   - f (2 \gamma + \mu)\right) \delta^2 > 0  \label{eqn:case2_noise} 
\end{align}
From~\eqref{eqn:ineq_1_noise},~\eqref{eqn:case1_noise} and~\eqref{eqn:case2_noise}, we get
\begin{align*}
    h_{t+1} \leq h_t + \eta^2_t K, \, \forall t \in \Z_{\geq 0}
\end{align*}
As $K > 0$ (cf.~\eqref{eqn:omega}), the above implies that (refer Lemma~\ref{lem:seq_conv} for the notation $(\cdot)_{+}$)
\begin{align*}
    (h_{t+1} - h_t)_{+} \leq \eta^2_t K, \, \forall t \in \Z_{\geq 0}
\end{align*}
As $\sum_{t = 0}^\infty \eta^2_t < \infty$ and $K < \infty$ (ref.~\eqref{eqn:omega}), thus $\sum_{t = 0}^\infty \eta^2_t K < \infty$. This implies that the infinite sum of the positive variance of the sequence $\{h_t\}$ is finite, i.e.
\begin{align*}
    \sum_{t = 0}^\infty (h_{t+1} - h_t)_{+} \leq \sum_{t = 0}^\infty \eta^2_t K < \infty
\end{align*}
As $h_{t} \geq 0, \, \forall t \in \Z_{\geq 0}$, the above implies that (cf. Lemma~\ref{lem:seq_conv})
\begin{align}
    h_t \underset{t \to \infty}{\longrightarrow} h_{\infty} < \infty \text{ and } \sum_{t = 0}^{\infty} (h_{t+1} - h_t)_{-} > -\infty \label{eqn:conv_noise}
\end{align}
where, operator $(\cdot)_{-}$ is same as defined in Lemma \ref{lem:seq_conv}. Now, 
\[h_{\infty} - h_0 = \sum_{t = 0}^\infty (h_{t+1}-h_t) \]
Thus, from \eqref{eqn:ineq_1_noise} we obtain,
\[h_{\infty} - h_0 \leq - 2 \sum_{t = 0}^\infty \eta_t \phi_t +  K \sum_{t = 0}^\infty  \eta^2_t\]
As $\sum_{t = 0}^\infty \eta^2_t < \infty$ and $K < \infty$, using \eqref{eqn:conv_noise} above implies
\begin{align}
    \sum_{t = 0}^\infty \eta_t \phi_t < \infty \label{eqn:contradict_noise}
\end{align}
Now, we show that $h_{\infty} = 0$ using reasoning by contradiction. Suppose there exists a positive real value $\beta$, such that $h_{\infty} = 2 \beta (2 \sqrt{\widehat{D}} + \sqrt{\beta})^2 > 0$. Note that for any positive real value $\epsilon$, there is a unique positive real value $\beta$ such that $ 2 \beta (2 \sqrt{\widehat{D}} + \sqrt{\beta})^2 = \epsilon$. \\
\noindent Thus, there exists a finite $\tau \in Z_{\geq 0}$ such that 
\[\mnorm{h_t - h_\infty} \leq \beta \left(2 \sqrt{\widehat{D}} + \sqrt{\beta}\right)^2, \quad \forall t \geq \tau\]
Thus (ref.~\eqref{eqn:ht_def_noise}),
\[ h_t = \psi\left( \norm{w^t - w^*}^2 \right) \geq \beta \left(2 \sqrt{\widehat{D}} + \sqrt{\beta}\right)^2 , \quad  \forall t \geq \tau\]
This implies that (ref. definition of $\psi(\cdot)$ in~\eqref{eqn:def_psi})
\[\left(\norm{w^t - w^*}^2 - \widehat{D} \right)^2 \geq \beta \left(2 \sqrt{\widehat{D}} + \sqrt{\beta}\right)^2,  \quad  \forall t \geq \tau\]
As $\norm{w^t - w^*}^2 \geq 0$ and $\beta > 0$, the above implies,
\[\norm{w^t - w^*}^2 \geq \widehat{D} + \sqrt{\beta} (2 \sqrt{\widehat{D}} + \sqrt{\beta}) = \left(\sqrt{\beta} + \sqrt{\widehat{D}}\right)^2, \quad  \forall t \geq \tau\]
\[\implies \norm{w^t - w^*} \geq \sqrt{\widehat{D}} + \sqrt{\beta}, \quad  \forall t \geq \tau \]
Thus, from~\eqref{eqn:case2_noise} we get,
\[ \phi_t \geq  \left(n \gamma   - f (2 \gamma + \mu)\right) \beta^2, \quad  \forall t \geq \tau\]
This implies that if condition \eqref{eqn:cond_2} is satisfied (i.e. $n \gamma  > f (2 \gamma + \mu)$) then, 
\[\sum_{t = \tau}^{\infty}\eta_t \phi_t \geq  \left(n \gamma   - f (2 \gamma + \mu)\right) \beta^2 \sum_{t = \tau}^{\infty}\eta_t = \infty, \quad \left( \text{as } \eta_t > 0,\,\forall t \in \Z_{\geq 0}, \text{ and } \sum_{t = 0}^{\infty}\eta_t = \infty \right) \]
The above contradicts~\eqref{eqn:contradict_noise}. Therefore, $h_{\infty} \not > 0$ and hence, 
\[ \lim_{t \to \infty}\norm{w^t - w^*} \leq \sqrt{\widehat{D}} \]
As $\widehat{D} = (D^*)^2$, the above implies that there exists a finite $\tau \in \Z_{\geq 0}$ such that
\[\norm{w^t - w^*} \leq D^*, \quad \forall t \geq \tau\]

\end{appendices}

\end{document}